\documentclass{article}

\usepackage[margin=3cm, bottom=3cm, footskip=1.5cm]{geometry}

\usepackage{amsmath,amssymb,amsthm}
\usepackage{algorithm}
\usepackage{algorithmicx,algpseudocode}
\usepackage[subrefformat=parens]{subcaption}
\usepackage[bbgreekl]{mathbbol}
\usepackage{graphicx}
\usepackage{bm}
\usepackage{url}

\newcommand{\E}{\mathbb{E}} % expectation
\newcommand{\PP}{\mathbb{P}} % probability
\newcommand{\QQ}{\mathbb{Q}} % probability in Q
\newcommand{\RR}{\mathbb{R}} % real values
\newcommand{\Cov}{\mathbb{COV}} % covariance matrix

\newtheorem{theorem}{Theorem}
\newtheorem{lemma}{Lemma}
\newtheorem{corollary}{Corollary}

\DeclareMathOperator*{\argmin}{\arg\!\min}

\title{Boltzmann machines and energy-based models}
\author{Takayuki Osogami\\IBM Research - Tokyo}
\date{osogami@jp.ibm.com}

\begin{document}

\maketitle

\begin{abstract}
We review Boltzmann machines and energy-based models.  A Boltzmann 
machine defines a probability distribution over binary-valued patterns.  
 One can learn parameters of a Boltzmann machine via gradient based 
approaches in a way that log likelihood of data is increased.  The 
gradient and Hessian of a Boltzmann machine admit beautiful 
mathematical representations, although computing them is in general 
intractable.  This intractability motivates approximate methods, 
including Gibbs sampler and contrastive divergence, and tractable 
alternatives, namely energy-based models.
\end{abstract}

\section{Introduction}

The Boltzmann machine has received considerable attention particularly
after the publication of the seminal paper by Hinton and Salakhutdinov
on autoencoder with stacked restricted Boltzmann machines
\cite{HinSal06}, which leads to today's success of and expectation to
deep learning \cite{DeepLearning_Sch,DeepLearning_GBC} as well as a wide
range of applications of Boltzmann machines such as collaborative
filtering \cite{SMH07}, classification of images and documents
\cite{classificationRBM}, and human choice \cite{RBMchoice,DeepChoice}.
The Boltzmann machine is a stochastic (generative) model that can
represent a probability distribution over binary patterns and others
(see Section~\ref{sec:BM:BM}).  The stochastic or generative capability
of the Boltzmann machine has not been fully exploited in today's deep
learning.  For further advancement of the field, it is important to
understand basics of the Boltzmann machine particularly from
probabilistic perspectives.  In this paper, we review fundamental
properties of the Boltzmann machines with particular emphasis on
probabilistic representations that allow intuitive interpretations in
terms of probabilities.

A core of this paper is in the learning rules based on gradients or
stochastic gradients for Boltzmann machines
(Section~\ref{sec:BM:generative}-Section~\ref{sec:BM:discriminative}).
These learning rules maximize the log likelihood of given dataset or
minimize the Kullback-Leibler (KL) divergence to a target
distribution.  In particular, Boltzmann machines admit concise
mathematical representations for its gradients and Hessians.  For
example, Hessians can be represented with covariance matrices.

The exact learning rules, however, turn out to be computationally 
intractable for general Boltzmann machines.  We then review approximate 
learning methods such as Gibbs sampler and contrastive divergence in 
Section~\ref{sec:BM:evaluate}.  

We also review other models that are related to the Boltzmann machine in 
Section~\ref{sec:BM:others}.  For example, the Markov random field is a 
generalization of the Boltzmann machine.  We also discuss how to deal 
with real valued distributions by modifying the Boltzmann machine.  

The intractability of exact learning of the Boltzmann machine motivates 
tractable energy-based learning.  Some of the approximate learning 
methods for the Boltzmann machine may be considered as a form of 
energy-based learning.  As a practical example, we review an 
energy-based model for face detection in Section~\ref{sec:BM:energy}.

This survey paper is based on a personal note prepared for the first
of the four parts of a tutorial given at the 26th International Joint
Conference on Artificial Intelligence (IJCAI-17) held in Melbourne,
Australia on August 21, 2017. See a tutorial
webpage\footnote{https://researcher.watson.ibm.com/researcher/view\_group.php?id
  =7834} for information about the tutorial.  A survey corresponding
to the third part of the tutorial (Boltzmann machines for time-series)
can be found in \cite{BMsurvey3}.

\section{The Boltzmann machine}
\label{sec:BM:BM}

A Boltzmann machine is a network of units that are connected to each
other (see Figure~\ref{fig:BM:BM}).  Let $N$ be the number of units.
Each unit takes a binary value (0 or 1).  Let $X_i$ be the random
variable representing the value of the $i$-th unit for $i\in[1,N]$.
We use a column vector $\boldsymbol{X}$ to denote the random values of the $N$
units.  The Boltzmann machine has two types of parameters: bias and
weight.  Let $b_i$ be the bias for the $i$-th unit for $i\in[1,N]$,
and let $w_{i,j}$ be the weight between unit $i$ and unit $j$ for
$(i,j)\in[1,N-1]\times[i+1,N]$.  We use a column vector $\mathbf{b}$
to denote the bias for all units and a matrix $\mathbf{W}$ to denote
the weight for all pairs of units.  Namely, the $(i,j)$-the element of
$\mathbf{W}$ is $w_{i,j}$.  We let $w_{i,j}=0$ for $i\ge j$ and for
the pair of units $(i,j)$ that are disconnected each other.  The
parameters are collectively denoted by
\begin{align}
\theta
& \equiv
(b_1, \ldots, b_N, w_{1,2}, \ldots, w_{N-1,N}),
\end{align}
which we also denote as $\theta=(\mathbf{b},\mathbf{W})$.

\begin{figure}[tb]
  \centering
  \includegraphics[width=0.5\linewidth]{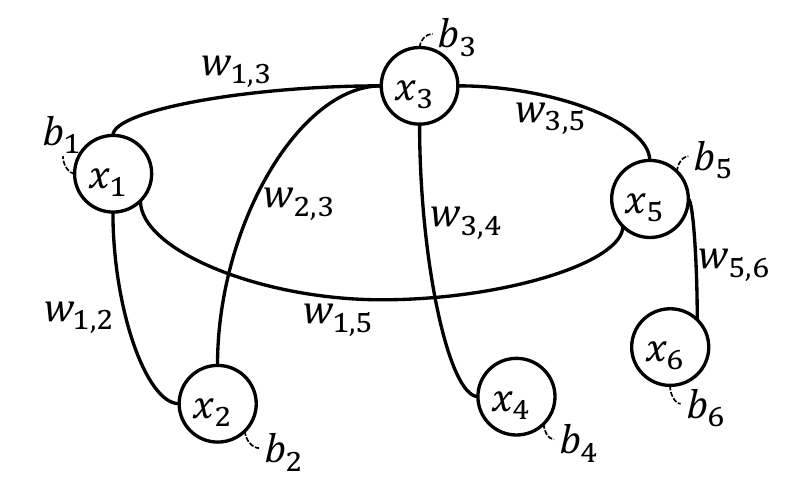}
  \caption{A Boltzmann machine.}
  \label{fig:BM:BM}
\end{figure}

The energy of the Boltzmann machine is defined by
\begin{align}
  E_\theta(\mathbf{x})
  & = - \sum_{i=1}^N b_i \, x_i - \sum_{i=1}^{N-1}\sum_{j=i+1}^N w_{i,j} \, x_i \, x_j \\
  & = - \mathbf{b}^\top \mathbf{x} - \mathbf{x}^\top \mathbf{W} \mathbf{x}.
  \label{eq:BM:energy}
\end{align}
From the energy, the Boltzmann machine defines the probability
distribution over binary patterns as follows:
\begin{align}
  \PP_\theta(\mathbf{x}) 
  = \frac{\exp\left( -E_\theta(\mathbf{x}) \right)}
  {\displaystyle\sum_{\mathbf{\tilde x}}\exp\left( -E_\theta(\mathbf{\tilde x}) \right)}
\label{eq:BM:prob}
\end{align}
where the summation with respect to $\mathbf{\tilde x}$ is over all of the possible $N$ bit binary values.
Namely, the higher the energy of a pattern $\mathbf{x}$, the less
likely that the $\mathbf{x}$ is generated.  For a moment, we do not
address the computational aspect of the denominator,
which involves a summation of $2^N$ terms.
This denominator is also known as a partition function:
\begin{align}
  Z \equiv \sum_{\mathbf{\tilde x}} \exp\left( -E_\theta(\mathbf{x}) \right).
\end{align}

A Boltzmann machine can be used to model the probability distribution,
$\PP_{\rm target}(\cdot)$, of target patterns.  Namely, by optimally
setting the values of $\theta$, we approximate $\PP_{\rm
  target}(\cdot)$ with $\PP_\theta(\cdot)$.  Here, some of the units
of the Boltzmann machine are allowed to be hidden, which means that
those units do not directly correspond to the target patterns (see
Figure~\ref{fig:BM:roles:hidden}).  The units that directly correspond
to the target patterns are called visible.  The primary purpose of the
hidden units is to allow particular dependency between visible units,
which cannot be represented solely with visible units.  The visible
units may be divided into input and output (see
Figure~\ref{fig:BM:roles:inout}).  Then the Boltzmann machine can be
used to model the conditional distribution of the output patterns given
an input pattern.

\begin{figure}[bt]
  \begin{minipage}[b]{0.33\linewidth}
    \includegraphics[width=\linewidth]{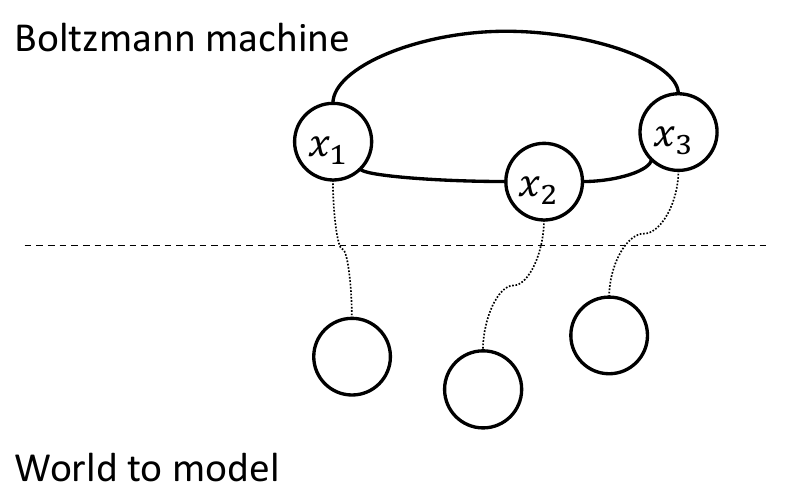}
    \subcaption{visible only}
    \label{fig:BM:roles:visible}
  \end{minipage}
  \begin{minipage}[b]{0.33\linewidth}
    \includegraphics[width=\linewidth]{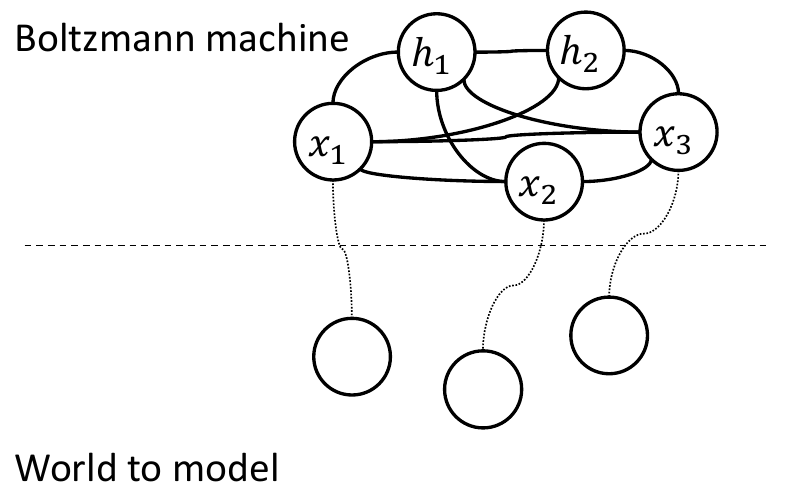}
    \subcaption{visible and hidden}
    \label{fig:BM:roles:hidden}
  \end{minipage}
  \begin{minipage}[b]{0.33\linewidth}
    \includegraphics[width=\linewidth]{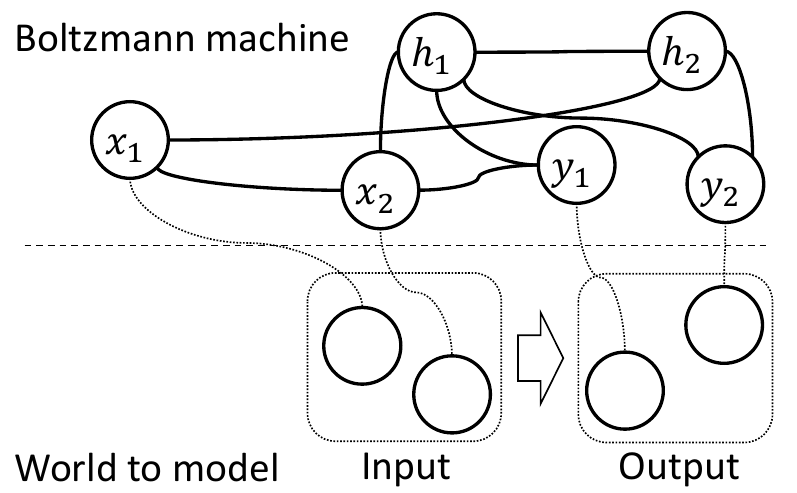}
    \subcaption{input and output}
    \label{fig:BM:roles:inout}
  \end{minipage}
  \caption{Boltzmann machines with hidden units, input, and output}
  \label{fig:BM:roles}
\end{figure}

%\section{Example applications for generative learning of spatial patterns}
%- recommendation (RBM, DPP)
%- choice model (RBM)
%- summary: document, video (DPP)

 \section{Learning a generative model}
\label{sec:BM:generative}

Now we consider the problem of optimally setting the values of
$\theta$ in a way that $\PP_\theta(\cdot)$ best approximates a given
$\PP_{\rm target}(\cdot)$.  Specifically, we seek to minimize the
Kullback-Leibler (KL) divergence from $\PP_\theta$ to $\PP_{\rm
  target}$ \cite{AHS85}:
\begin{align}
  {\rm KL}(\PP_{\rm target} \,||\, \PP_\theta)
  & \equiv
  \sum_{\mathbf{\tilde x}}
  \PP_{\rm target}(\mathbf{\tilde x})
  \log \frac{\PP_{\rm target}(\mathbf{\tilde x})}{\PP_\theta(\mathbf{\tilde x})}\\
  & =   \sum_{\mathbf{\tilde x}}
  \PP_{\rm target}(\mathbf{\tilde x})
  \log \PP_{\rm target}(\mathbf{\tilde x})
  - \sum_{\mathbf{\tilde x}}
  \PP_{\rm target}(\mathbf{\tilde x})
  \log \PP_\theta(\mathbf{\tilde x}).
  \label{eq:BM:KL}
\end{align}
The first term of \eqref{eq:BM:KL} is independent of $\theta$.  It
thus suffices to maximize the negation of the second term:
\begin{align}
  f(\theta)
  & \equiv
  \sum_{\mathbf{\tilde x}}
  \PP_{\rm target}(\mathbf{\tilde x})
  \log \PP_\theta(\mathbf{\tilde x}).
  \label{eq:BM:f}
\end{align}

A special case of $\PP_{\rm target}$ is the empirical distribution
of the patterns in a given training dataset:
\begin{align}
  \mathcal{D} = \{ \mathbf{x}^{(d)} \}_{d=1}^D,
\end{align}
where $D$ is the number of the patterns in $\mathcal{D}$, Then the
objective function \eqref{eq:BM:f} becomes
\begin{align}
  f(\theta)
  & = \frac{1}{D} \sum_{\mathbf{x} \in \mathcal{D}} \log \PP_\theta(\mathbf{x}) \label{eq:BM:log-likelihood}\\
  & = \frac{1}{D} \log \prod_{\mathbf{x} \in \mathcal{D}} \PP_\theta(\mathbf{x}),
\end{align}
which is the log-likelihood of $\mathcal{D}$ with respect to
$\PP_\theta$ when multiplied by $D$.  By defining
\begin{align}
  \PP_\theta(\mathcal{D}) \equiv \prod_{\mathbf{x} \in \mathcal{D}} \PP_\theta(\mathbf{x}),
\end{align}
we can represent $f(\theta)$ as follows:
\begin{align}
  f(\theta)
  & = \frac{1}{D} \log \PP_\theta(\mathcal{D}).
\end{align}

To find the optimal values of $\theta$, we take the gradient of
$f(\theta)$ with respect to $\theta$:
\begin{align}
  \nabla f(\theta)
  = \sum_{\mathbf{\tilde x}} 
  \PP_{\rm target}(\mathbf{\tilde x}) \,
  \nabla \log \PP_\theta(\mathbf{\tilde x}).
  \label{eq:BM:grad}
\end{align}

\subsection{All of the units are visible}

We start with the simplest case where all of the units are visible
(see Figure~\ref{fig:BM:roles:visible}).  Then the energy of
the Boltzmann machine is simply given by \eqref{eq:BM:energy}, and the
probability distribution is given by \eqref{eq:BM:prob}.

\subsubsection{Gradient}
We will derive a specific representation of $\nabla \log
\PP_\theta(\mathbf{x})$ to examine the form of $\nabla
f(\theta)$ in this case:
\begin{align}
  \nabla \log \PP_\theta(\mathbf{x})
  & = \nabla \log \frac{\exp\left( -E_\theta(\mathbf{x}) \right)}
         {\displaystyle\sum_{\mathbf{\hat x}}\exp\left( -E_\theta(\mathbf{\hat x}) \right)} 
  \label{eq:gradLL} \\
  & = -\nabla E_\theta(\mathbf{x})
  - \nabla \log \sum_{\mathbf{\hat x}}\exp\left( -E_\theta(\mathbf{\hat x}) \right) \\
  & = -\nabla E_\theta(\mathbf{x})
  + \frac{\displaystyle\sum_{\mathbf{\hat x}}\exp\left( -E_\theta(\mathbf{\hat x}) \right) \, \nabla E_\theta(\mathbf{\hat x})}
         {\displaystyle\sum_{\mathbf{\tilde x}}\exp\left( -E_\theta(\mathbf{\tilde x}) \right)} 
  \label{eq:BM:grad-log-prev} \\
  & = -\nabla E_\theta(\mathbf{x})
         + \sum_{\mathbf{\hat x}} \PP_\theta(\mathbf{\hat x}) \, \nabla E_\theta(\mathbf{\hat x}),
  \label{eq:BM:grad-log}
\end{align}
where the summation with respect to $\mathbf{\hat x}$ is over all of
the possible binary patterns, similar to the summation with respect to
$\mathbf{\tilde x}$.  Here, \eqref{eq:BM:grad-log} follows from \eqref{eq:BM:prob} and \eqref{eq:BM:grad-log-prev}.

Plugging the last expression into \eqref{eq:BM:grad}, we obtain
\begin{align}
  \nabla f(\theta)
  & = - \sum_{\mathbf{\tilde x}} 
  \PP_{\rm target}(\mathbf{\tilde x}) \,
  \nabla E_\theta(\mathbf{\tilde x})  
  + \sum_{\mathbf{\tilde x}} 
  \PP_{\rm target}(\mathbf{\tilde x}) \, \sum_{\mathbf{\hat x}} \PP_\theta(\mathbf{\hat x}) \, \nabla E_\theta(\mathbf{\hat x}) \\
  & = - \sum_{\mathbf{\tilde x}} 
  \PP_{\rm target}(\mathbf{\tilde x}) \,
  \nabla E_\theta(\mathbf{\tilde x})  
  + \sum_{\mathbf{\hat x}} \PP_\theta(\mathbf{\hat x}) \, \nabla E_\theta(\mathbf{\hat x})
  \label{eq:BM:grad-general-1}\\
  & = - \sum_{\mathbf{\tilde x}} 
  \left( \PP_{\rm target}(\mathbf{\tilde x}) - \PP_\theta(\mathbf{\tilde x}) \right)
  \nabla E_\theta(\mathbf{\tilde x}).
  \label{eq:BM:grad-general}
\end{align}

The last expression allows intuitive interpretation of a
gradient-based method for increasing the value of $f(\theta)$:
\begin{align}
  \theta \leftarrow \theta + \eta \, \nabla f(\theta),
\end{align}
where $\eta$ is the learning rate (or the step size).  Namely, for
each pattern $\mathbf{\tilde x}$, we compare
$\PP_\theta(\mathbf{\tilde x})$ against $\PP_{\rm
  target}(\mathbf{\tilde x})$.  If $\PP_\theta(\mathbf{\tilde x})$ is
greater than $\PP_{\rm target}(\mathbf{\tilde x})$, we update $\theta$
in a way that it increases the energy $E_\theta(\mathbf{\tilde x})$ so
that the $\mathbf{\tilde x}$ becomes less likely to be generated with
$\PP_\theta$.  If $\PP_\theta(\mathbf{\tilde x})$ is smaller than
$\PP_{\rm target}(\mathbf{\tilde x})$, we update $\theta$ in a way
that $E_\theta(\mathbf{\tilde x})$ decreases.

We will also write \eqref{eq:BM:grad-general-1} as follows:
\begin{align}
\nabla f(\theta)
=
- \E_{\rm target}\left[ \nabla E_\theta(\boldsymbol{X}) \right]
+ \E_\theta\left[ \nabla E_\theta(\boldsymbol{X}) \right],
\label{eq:BM:grad-in-E}
\end{align}
where $\E_{\rm target}[\cdot]$ is the expectation with respect to
$\PP_{\rm target}$, $\E_\theta[\cdot]$ is the expectation with respect
to $\PP_\theta$, and $\boldsymbol{X}$ is the vector of random variables
denoting the values of the $N$ units.  Note that the expression of the
gradient in \eqref{eq:BM:grad-in-E} holds for any form of energy, as
long as the energy is used to define the probability as in
\eqref{eq:BM:prob}.

Now we take into account the specific form of the energy given by
\eqref{eq:BM:energy}.  Taking the derivative with respect to each
parameter, we obtain
\begin{align}
  \frac{\partial}{\partial b_i} E_\theta(\mathbf{x}) & = - x_i \\
  \frac{\partial}{\partial w_{i,j}} E_\theta(\mathbf{x}) & = - x_i \, x_j
  \label{eq:BM:partial-w-0}
\end{align}
for $i\in[1,N]$ and $(i,j)\in[1,N-1]\times[i+1,N]$.  From \eqref{eq:BM:grad-in-E}, we then find
\begin{align}
  \frac{\partial}{\partial b_i} f(\theta)
  & = \E_{\rm target}[X_i] -  \E_\theta [X_i] \\
  \frac{\partial}{\partial w_{i,j}} f(\theta)
  & = \E_{\rm target}[X_i\,X_j] -  \E_\theta [X_i\,X_j],
  \label{eq:BM:partial-w}
\end{align}
where $X_i$ is the random
variable denoting the value of the $i$-th unit for each $i\in[1,N]$.
Notice that the expected value of $X_i$ is the same as the probability of
$X_i=1$, because $X_i$ is binary.  In general, exact evaluation of
$\E_\theta[X_i]$ or $\E_\theta[X_i\,X_j]$ is
computationally intractable, but we will not be concerned with this
computational aspect until Section~\ref{sec:BM:evaluate}.

A gradient ascent method is thus to iteratively update the parameters as follows:
\begin{align}
  b_i
  & \leftarrow
  b_i + \eta \, \left(\E_{\rm target}[X_i] -  \E_\theta[X_i]\right) 
\label{eq:BM:grad_bi}\\
  w_{i,j}
  & \leftarrow
  w_{i,j} + \eta \, \left(\E_{\rm target}[X_i\,X_j] -  \E_\theta[X_i\,X_j]\right)
\label{eq:BM:grad_wij}
\end{align}
for $i\in[1,N]$ and $(i,j)\in[1,N-1]\times[i+1,N]$.  Intuitively,
$b_i$ controls how likely that the $i$-th unit takes the value 1, and
$w_{i,j}$ controls how likely that the $i$-th unit and the $j$-th unit
simultaneously take the value 1.  For example, when
$\E_\theta[X_i\,X_j]$ is smaller than $\E_{\rm target}[X_i\,X_j]$, we
increase $w_{i,j}$ to increase $\E_\theta[X_i\,X_j]$.  This form of learning rule appears frequently in the context of Boltzmann machines.  Namely, we compare our prediction $\E_\theta[\cdot]$ against the target $E_{\rm target}[\cdot]$ and update $\theta$ in a way that $\E_\theta[\cdot]$ gets closer to $\E_{\rm target}[\cdot]$.

\subsubsection{Stochastic gradient}
We now rewrite \eqref{eq:BM:grad-general-1} as follows:
\begin{align}
  \nabla f(\theta)
  = \sum_{\mathbf{\tilde x}} \PP_{\rm target}(\mathbf{\tilde x})
  \left( 
 - \nabla E_\theta(\mathbf{\tilde x}) 
 + \E_\theta\left[ \nabla E_\theta(\boldsymbol{X}) \right]
  \right).
\end{align}
Namely, $\nabla f(\theta)$ is given by the expected value of $- \nabla
E_\theta(\boldsymbol{X}) + \E_\theta\left[ \nabla E_\theta(\boldsymbol{X})
  \right]$, where the first $\boldsymbol{X}$ is distributed with respect
to $\PP_{\rm target}$.  Recall that the second term is an expectation
with respect to $\PP_\theta$.  This suggests stochastic gradient
methods \cite{SGD,Adam,AdaGrad,RMSProp,momentum}.  At each step, we
sample a pattern $\boldsymbol{X}(\omega)$ according to $\PP_{\rm target}$
and update $\theta$ according to the stochastic gradient:
\begin{align}
  \theta & \leftarrow \theta + \eta \, g_\theta(\omega),
  \label{eq:BM:SGD}
\intertext{where}
  g_\theta(\omega)
  & \equiv 
  - \nabla E_\theta(\boldsymbol{X}(\omega))
  + \E_\theta\left[ \nabla E_\theta(\boldsymbol{X}) \right].
  \label{eq:BM:SG}
\end{align}
When the target distribution is the empirical distribution given by
the training data ${\cal D}$, we only need to take a sample $\boldsymbol{X}(\omega)$
from ${\cal D}$ uniformly at random.

The stochastic gradient method based on
\eqref{eq:BM:SGD}-\eqref{eq:BM:SG} allows intuitive interpretation.
At each step, we sample a pattern according to the target distribution
(or from the training data) and update $\theta$ in a way that the
energy of the sampled pattern is reduced.  At the same time, the
energy of every pattern is increased, where the amount of the increase
is proportional to the probability for the Boltzmann machine with
the latest parameter $\theta$ to generate that pattern (see Figure~\ref{fig:BM:SGD}).

\begin{figure}
\centering
  \includegraphics[width=0.5\linewidth]{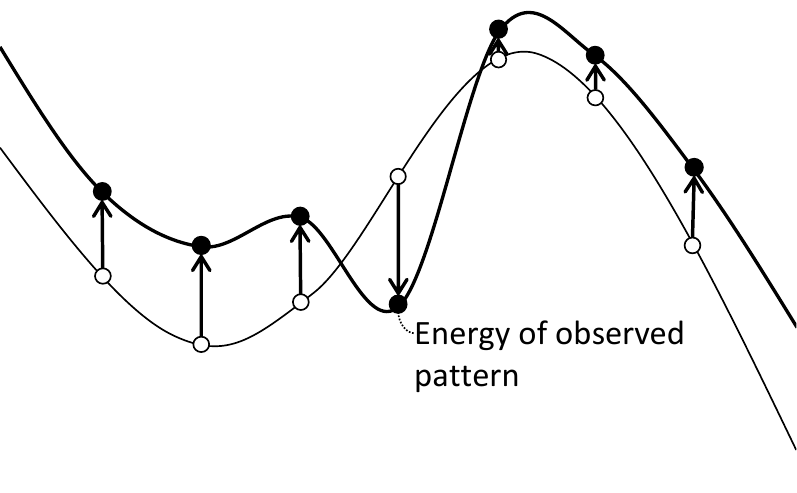}
  \caption{How the energy is updated}
  \label{fig:BM:SGD}
\end{figure}

Taking into account the specific form of the energy given by
\eqref{eq:BM:energy}, we can derive the specific form of the stochastic gradient:
\begin{align}
  \frac{\partial}{\partial b_i} E_\theta(\mathbf{x})(\omega)
  & = X_i(\omega) - \E_\theta[X_i] \\
  \frac{\partial}{\partial w_{i,j}} E_\theta(\mathbf{x})(\omega)
  & = X_i(\omega) \, X_j(\omega) - \E_\theta[X_i\,X_j],
\end{align}
which suggests a stochastic gradient method of iteratively sampling a
pattern $\boldsymbol{X}(\omega)$ according to the target probability
distribution and updating the parameters as follows:
\begin{align}
  b_i
  & \leftarrow
  b_i + \eta \, \left( X_i(\omega) -  \E_\theta[X_i] \right) \\
  w_{i,j}
  & \leftarrow
  w_{i,j} + \eta \, \left( X_i(\omega)\,X_j(\omega) -  \E_\theta[X_i\,X_j] \right)
  \label{eq:BM:update-w}
\end{align}
for $i\in[1,N]$ and $(i,j)\in[1,N-1]\times[i+1,N]$.

\subsubsection{Giving theoretical foundation for Hebb's rule}

The learning rule of \eqref{eq:BM:update-w} has a paramount importance
of providing a theoretical foundation for Hebb's rule of
learning in biological neural networks~\cite{Hebb}:
\begin{quote}
\textit{When an axon of cell A is near enough to excite a cell B and
repeatedly or persistently takes part in firing it, some growth
process or metabolic change takes place in one or both cells such that
A's efficiency, as one of the cells firing B, is increased.}
\end{quote}
In short, ``neurons wire together if they fire together'' \cite{shortHebb}.

A unit of a Boltzmann machine corresponds to a neuron, and $X_i=1$
means that the $i$-th neuron fires.  When two neurons, $i$ and $j$,
fire ($X_i(\omega)\,X_j(\omega)=1$), the wiring weight $w_{i,j}$
between the two neurons gets stronger according to
\eqref{eq:BM:update-w}.  Here, notice that we have $0<\E_\theta[X_i\,X_j]
<1$ as long as the values of $\theta$ are finite.

The learning rule of the Boltzmann machine also involves a mechanism
beyond what is suggested by Hebb's rule.  Namely, the amount of the
change in $w_{i,j}$ when the two neurons ($i$ and $j$) fire depends on
how likely those two neurons fire according to $\PP_\theta$
at
that time.  More specifically, it the two neurons are already expected
to fire together ({\it i.e.}, $\E_\theta[X_i\,X_j]\approx
1$), we increase $w_{i,j}$ only by a small amount ({\it i.e.}, $\eta
\, \left(1 - \E_\theta[X_i\,X_j]\right)$) even if the two
neurons fire together ({\it i.e.}, $X_i(\omega)\,X_j(\omega)=1$).

Without this additional term ($-\E_\theta[X_i]$ or
$-\E_\theta[X_i\,X_j]$) in \eqref{eq:BM:update-w}, all of the
parameters monotonically increases.  If $X_i=1$ with nonzero
probability in $\PP_{\rm target}$, then $b_i$ diverges to $\infty$
almost surely.  Otherwise, $b_i$ stays unchanged from the initial
value.  Likewise, if $X_i\,X_j=1$ with nonzero probability in
$\PP_{\rm target}$, then $w_{i,j}$ diverges to $\infty$ almost surely.
Otherwise, $w_{i,j}$ stays unchanged.

What is important is that this additional term is formally derived
instead of being introduced in an ad hoc manner.  Specifically, the
learning rule is derived from a stochastic model ({\it i.e.}, a
Boltzmann machine) and an objective function ({\it i.e.}, minimizing
the KL divergence to the target distribution or maximizing the
log-likelihood of training data) by taking the gradient with respect
to the parameters.

\subsubsection{Hessian}

We now derive the Hessian of $f(\theta)$ to examine its landscape.
Starting from the expression in \eqref{eq:BM:partial-w}, we obtain
\begin{align}
  \frac{\partial}{\partial w_{k,\ell}} \frac{\partial}{\partial w_{i,j}} f(\theta)
  & = - \frac{\partial}{\partial w_{k,\ell}} \E_\theta[X_i\,X_j] \\
  & = - \sum_{\mathbf{\tilde x}} \tilde x_i \, \tilde x_j \,
  \frac{\partial}{\partial w_{k,\ell}} \PP_\theta(\mathbf{\tilde x}) \label{eq:BM:partial2} \\
  & = - \sum_{\mathbf{\tilde x}} \tilde x_i \, \tilde x_j \,
  \PP_\theta(\mathbf{\tilde x}) \,
  \frac{\partial}{\partial w_{k,\ell}} \log \PP_\theta(\mathbf{\tilde x}) \\
  & = \left(\sum_{\mathbf{\tilde x}} \PP_\theta(\mathbf{\tilde x}) \, \tilde x_i \, \tilde x_j\right)
  \left(\sum_{\mathbf{\tilde x}} \PP_\theta(\mathbf{\tilde x}) \, \tilde x_k \, \tilde x_\ell\right)
  - \sum_{\mathbf{\tilde x}} \PP_\theta(\mathbf{\tilde x}) \, \tilde x_i \, \tilde x_j \, \tilde x_k \, \tilde x_\ell,
\end{align}
where the last expression is obtained from \eqref{eq:BM:grad-log} and
\eqref{eq:BM:partial-w-0}.  The last expression consists of expectations
with respect to $\PP_\theta$ and can be represented conveniently as
follows:
 \begin{align}
  \frac{\partial}{\partial w_{k,\ell}} \frac{\partial}{\partial w_{i,j}}
  f(\theta) & = \E_\theta[X_i\,X_j] \, \E_\theta[X_k\,X_\ell] -
  \E_\theta[X_i\,X_j\,X_k\,X_\ell] \\ & = -\Cov_\theta[ X_i\,X_j,
  X_k\,X_\ell ], \label{eq:BM:Cov}
 \end{align}
where $\Cov_\theta[A,B]$ denotes the covariance between random
variables $A$ and $B$ with respect to $\PP_\theta$.  Likewise, we have
\begin{align}
  \frac{\partial}{\partial b_{k}} \frac{\partial}{\partial w_{i,j}} f(\theta)
  & = -\Cov_\theta[ X_i\,X_j, X_k ] \\
  \frac{\partial}{\partial b_j} \frac{\partial}{\partial b_i} f(\theta)
  & = -\Cov_\theta[ X_i, X_j ].
\end{align}

Therefore, the Hessian of $f(\theta)$ is a covariance matrix:
\begin{align}
  \nabla^2 f(\theta)
  = - \Cov_\theta[ X_1,\ldots,X_N,X_1\,X_2,\ldots,X_{N-1}\,X_N ],
\end{align}
where we use $\Cov_\theta$ to denote a covariance matrix with respect
to $\PP_\theta$.  When $\theta$ is finite, this covariance matrix is
positive semidefinite, and $f(\theta)$ is concave.  This justifies
(stochastic) gradient based approaches to optimizing $\theta$.  This
concavity has been known \cite{BM_scholarpedia}, but I am not aware of
the literature that explicitly represent the Hessian with a
covariance matrix.

\subsubsection{Summary}

Consider a Boltzmann machine with parameters 
$\theta=(\mathbf{b},\mathbf{W})$.  When all of the $N$ 
units of the Boltzmann machine are visible, the Boltzmann machine 
defines a probability distribution $\PP_\theta$ of $N$-bit binary 
patterns by
\begin{align}
  \PP_\theta(\mathbf{x})
  =
  \frac{\exp\left( -E_\theta(\mathbf{x}) \right)}
  {\displaystyle\sum_{\mathbf{\tilde x}} \exp\left( -E_\theta(\mathbf{\tilde x}) \right)},
\end{align}
where the energy is 
\begin{align}
  E_\theta(\mathbf{x})
  \equiv
  - \mathbf{b}^\top \mathbf{x} - \mathbf{x}^\top\mathbf{W}\,\mathbf{x}.
  \label{eq:BM:energy-visible}
\end{align}

The KL divergence from $\PP_\theta$ to $\PP_{\rm target}$ can be
minimized (or the log-likelihood of the target data having the
empirical distribution $\PP_{\rm target}$ can be maximized) by
maximizing
\begin{align}
  f(\theta)
  \equiv
  \E_{\rm target}\left[ \log \PP_\theta(\boldsymbol{X}) \right].
\end{align}
The gradient and the Hessian of $f(\theta)$ is given by
\begin{align}
  \nabla f(\theta)
  & = \E_{\rm target}[\boldsymbol{S}] - \E_\theta[\boldsymbol{S}]
  \label{eq:BM:gradH-visible}\\
  \nabla^2 f(\theta)
  & = - \Cov_\theta(\boldsymbol{S}),
\end{align}
where $\boldsymbol{S}$ denotes the vector of the random variables
representing the value of a unit or the product of the values of a
pair of units:
\begin{align}
  \boldsymbol{S} = (X_1, \ldots, X_N, X_1\,X_2, \ldots, X_{N-1}\,X_N).
  \label{eq:BM:S-visible}
\end{align}

\subsection{Some of the units are hidden}

In this section, we consider Boltzmann machines that have both visible units and 
hidden units.  Let $N$ be the number of visible units and $M$ be the 
number of hidden units.

\subsubsection{Necessity of hidden units}

We first study the necessity of hidden units \cite{AHS85}.  The
Boltzmann machine with $N$ units have
\begin{align}
  N + \frac{1}{2} N \, (N-1) = \frac{1}{2} N \, (N+1)
\end{align}
parameters.  This Boltzmann machine is used to model $N$-bit binary
patterns.  There are $2^N$ possible $N$-bit binary patterns, and the general distribution
of $N$-bit patterns assigns probabilities to those $2^N$ patterns.  We need
\begin{align}
  2^N - 1
\end{align}
parameters to characterize this general distribution.

The number of parameters of the Boltzmann machine is smaller than the
number of parameters needed to characterize the general distribution as
long as $N>2$.  This suggests that the probability distribution that
can be represented by the Boltzmann machine is limited.  One way to
extend the flexibility of the Boltzmann machine is the use of hidden
units.

\subsubsection{Free energy}

Let $\mathbf{x}$ denote the visible values ({\it i.e.}, the values of visible units),
$\mathbf{h}$ denote the hidden values,
and $(\mathbf{x},\mathbf{h})$ denote the values of all units.
We write the marginal probability
distribution of the visible values as follows:
\begin{align}
  \PP_\theta(\mathbf{x})
  & = \sum_{\mathbf{\tilde h}} \PP_\theta(\mathbf{x},\mathbf{\tilde h}),
\end{align}
where the summation is over all of the possible binary patterns of the hidden values, and
\begin{align}
  \PP_\theta(\mathbf{x},\mathbf{h})
  & = \frac{\exp\left( -E_\theta(\mathbf{x},\mathbf{h}) \right)}
     {\displaystyle\sum_{\mathbf{\tilde x},\mathbf{\tilde h}} \exp\left( -E_\theta(\mathbf{\tilde x},\mathbf{\tilde h}) \right)}.
\end{align}
Here, we write energy as follows:
\begin{align}
  E_\theta(\mathbf{x},\mathbf{h})
  & = -\mathbf{b}^\top \, {\mathbf{x}\choose\mathbf{h}}
  - \left(\mathbf{x}^\top,\mathbf{h}^\top\right) \, \mathbf{W} \, {\mathbf{x}\choose\mathbf{h}}
  \label{eq:BM:energy-hidden}\\
  & = -(\mathbf{b}^{\rm V})^\top \, \mathbf{x} - (\mathbf{b}^{\rm H})^\top \, \mathbf{h}
  -\mathbf{x}^\top \, \mathbf{W}^{\rm VV} \, \mathbf{x} 
  -\mathbf{x}^\top \, \mathbf{W}^{\rm VH} \, \mathbf{h} 
  -\mathbf{h}^\top \, \mathbf{W}^{\rm HH} \, \mathbf{h}.
\label{eq:BM:energy-hidden-2}
\end{align}

Now, we define free energy as follows:
\begin{align}
  F_\theta(\mathbf{x})
 & \equiv - \log \sum_{\mathbf{\tilde h}} \exp\left( -E_\theta(\mathbf{x},\mathbf{\tilde h} )\right).
 \label{eq:BM:freeenergy}
\end{align}
We can then represent $\PP_\theta(\mathbf{x})$ in a way similar to
the case where all of the units are visible, replacing energy with
free energy:
\begin{align}
  \PP_\theta(\mathbf{x})
  & = \sum_{\mathbf{\tilde h}} \PP_\theta(\mathbf{x},\mathbf{\tilde h}) \\
  & = \frac{\displaystyle\sum_{\mathbf{\tilde h}} \exp\left( -E_\theta(\mathbf{x},\mathbf{\tilde h}) \right)}
     {\displaystyle\sum_{\mathbf{\tilde x},\mathbf{\tilde h}} \exp\left( -E_\theta(\mathbf{\tilde x},\mathbf{\tilde h}) \right)} \\
  & = \frac{\exp\left( -F_\theta(\mathbf{x}) \right)}
     {\displaystyle\sum_{\mathbf{\tilde x}} \exp\left( -F_\theta(\mathbf{\tilde x}) \right)}.
\end{align}

\subsubsection{Gradient}

In \eqref{eq:BM:grad-general-1}, we simply replace energy with free
energy to obtain the gradient of our objective function when some of
the units are hidden:
\begin{align}
  \nabla f(\theta)
  & = - \sum_{\mathbf{\tilde x}} 
  \PP_{\rm target}(\mathbf{\tilde x}) \,
  \nabla F_\theta(\mathbf{\tilde x})  
  + \sum_{\mathbf{\tilde x}} \PP_\theta(\mathbf{\tilde x}) \, \nabla F_\theta(\mathbf{\tilde x}) \\
  & = - \E_{\rm target}\left[ \nabla F_\theta(\boldsymbol{X}) \right]
	+ \E_\theta\left[ \nabla F_\theta(\boldsymbol{X}) \right]
\end{align}

What we then need is the gradient of free energy:
\begin{align}
  \nabla F_\theta(\mathbf{x})
  & = - \nabla \log \sum_{\mathbf{\tilde h}} \exp\left( -E_\theta(\mathbf{x},\mathbf{\tilde h}) \right) \\
  & = \frac{\displaystyle\sum_{\mathbf{\tilde h}} \exp\left( -E_\theta(\mathbf{x},\mathbf{\tilde h}) \right) \, \nabla E_\theta(\mathbf{x},\mathbf{\tilde h})}
  {\displaystyle\sum_{\mathbf{\tilde h}} \exp\left( -E_\theta(\mathbf{x},\mathbf{\tilde h}) \right)} \\
  & = \sum_{\mathbf{\tilde h}} \PP_\theta(\mathbf{\tilde h} \,|\, \mathbf{x}) \,
  \nabla E_\theta(\mathbf{x},\mathbf{\tilde h}),
\label{eq:BM:grad_free_energy}
\end{align}
where $\PP_\theta(\mathbf{h} \,|\, \mathbf{x})$ is the conditional
probability that the hidden values are $\mathbf{h}$ given
that the visible values are $\mathbf{x}$:
\begin{align}
  \PP_\theta(\mathbf{h} \,|\, \mathbf{x})
  & \equiv \frac{\exp\left( -E_\theta(\mathbf{x},\mathbf{h}) \right)}
  {\displaystyle\sum_{\mathbf{\tilde h}} \exp\left( -E_\theta(\mathbf{x},\mathbf{\tilde h}) \right)} 
\label{eq:BM:conditional-probability}\\
  & = 
    \frac{\exp\left( -E_\theta(\mathbf{x},\mathbf{h}) \right)}
         {\displaystyle\sum_{\mathbf{\tilde x},\mathbf{\tilde h}} \exp\left( -E_\theta(\mathbf{x},\mathbf{\tilde h}) \right)}
    \,
    \frac{\displaystyle\sum_{\mathbf{\tilde x},\mathbf{\tilde h}} \exp\left( -E_\theta(\mathbf{x},\mathbf{\tilde h}) \right)}
              {\displaystyle\sum_{\mathbf{\tilde h}} \exp\left( -E_\theta(\mathbf{x},\mathbf{\tilde h}) \right)}
   \\
  & = \frac{\PP_\theta(\mathbf{x},\mathbf{h})}
  {\displaystyle\sum_{\mathbf{\tilde h}} \PP_\theta(\mathbf{x},\mathbf{\tilde h})} \\
  & = \frac{\PP_\theta(\mathbf{x},\mathbf{h})}
  {\PP_\theta(\mathbf{x})}.
\end{align}
Observe in \eqref{eq:BM:grad_free_energy} that the gradient of free energy 
is expected gradient of energy, where the expectation is with respect to the 
conditional distribution of hidden values given the visible values.

We thus obtain
\begin{align}
  \nabla f(\theta)
  & = - \sum_{\mathbf{\tilde x}} 
  \PP_{\rm target}(\mathbf{\tilde x}) \,
  \sum_{\mathbf{\tilde h}} \PP_\theta(\mathbf{\tilde h} \,|\, \mathbf{\tilde x}) \,
  \nabla E_\theta(\mathbf{\tilde x},\mathbf{\tilde h})
  + \sum_{\mathbf{\tilde x}} \PP_\theta(\mathbf{\tilde x}) \,
  \sum_{\mathbf{\tilde h}} \PP_\theta(\mathbf{\tilde h} \,|\, \mathbf{\tilde x}) \,
  \nabla E_\theta(\mathbf{\tilde x},\mathbf{\tilde h}) \\
  & = - \sum_{\mathbf{\tilde x},\mathbf{\tilde h}}
  \PP_{\rm target}(\mathbf{\tilde x}) \,  
  \PP_\theta(\mathbf{\tilde h} \,|\, \mathbf{\tilde x}) \,
  \nabla E_\theta(\mathbf{\tilde x},\mathbf{\tilde h})
  + \sum_{\mathbf{\tilde x},\mathbf{\tilde h}} \PP_\theta(\mathbf{\tilde x},\mathbf{\tilde h}) \,
  \nabla E_\theta(\mathbf{\tilde x},\mathbf{\tilde h}).
  \label{eq:BM:gradH-hidden-2}
\end{align}

The first term in the last expression (except the minus sign) is the
expected value of the gradient of energy, where the expectation is with
respect to the distribution defined with $\PP_{\rm target}$ and
$\PP_\theta$.  Specifically, the visible values follow $\PP_{\rm
target}$, and given the visible values, $\mathbf{x}$, the hidden values
follow $\PP_\theta(\cdot\,|\,\mathbf{x})$.  We will write this
expectation with $\E_{\rm target}\left[ \E_\theta[\cdot\,|\,
\boldsymbol{X}] \right]$.  The second term is expectation with respect
to $\PP_\theta$, which we denote with $\E_\theta[\cdot]$.  Because the
energy \eqref{eq:BM:energy-hidden} has the form equivalent to
\eqref{eq:BM:energy-visible}, $\nabla f(\theta)$ can then be represented
analogously to \eqref{eq:BM:gradH-visible}:
\begin{align}
  \nabla f(\theta)
  & = - \E_{\rm target}\left[ \E_\theta\left[ \nabla E_\theta(\boldsymbol{X}, \boldsymbol{H}) \,|\, \boldsymbol{X} \right]\right] 
 + \E_\theta\left[ \nabla E_\theta(\boldsymbol{X}, \boldsymbol{H}) \right]
 \label{eq:BM:generative:grad}\\
  & = \E_{\rm target}\left[ \E_\theta[ \boldsymbol{S} \,|\, \boldsymbol{X}] \right] - \E_\theta[\boldsymbol{S}],
\end{align}
where $\boldsymbol{X}$ is the vector of random values of the visible units, 
$\boldsymbol{H}$ is the vector of $H_i$ for $i\in[1,N]$,
and $\boldsymbol{S}$ is defined analogously to \eqref{eq:BM:S-visible} for all of the (visible or hidden) units:
\begin{align}
  \boldsymbol{S} = (U_1, \ldots, U_{N+M}, U_1\,U_2, \ldots, U_{N+M-1}\,U_{N+M}),
  \label{eq:BM:S}
\end{align}
where $U_i \equiv X_i$ for $i\in[1,N]$, and $U_{N+i} \equiv H_i$ for 
$i\in[1,M]$, where $H_i$ is the random variable denoting the 
$i$-th hidden value.

\subsubsection{Stochastic gradient}

The expression with \eqref{eq:BM:gradH-hidden-2} suggests stochastic
gradient analogous to \eqref{eq:BM:SGD}-\eqref{eq:BM:SG}.  Observe
that $\nabla f(\theta)$ can be represented as
\begin{align}
  \nabla f(\theta)
  & = \sum_{\mathbf{\tilde x}} \PP_{\rm target}(\mathbf{\tilde x}) \,
  \left(
  \E_\theta\left[ \nabla E_\theta(\boldsymbol{X},\boldsymbol{H}) \right]
  -
  \E_\theta\left[ \nabla E_\theta(\mathbf{\tilde x},\boldsymbol{H}) \,|\, \mathbf{\tilde x} \right]
  \right),
\end{align}
where $\E_\theta\left[ \nabla E_\theta(\boldsymbol{X},\boldsymbol{H}) \right]$ is
the expected value of the gradient of the energy when both 
visible values and hidden values follow $\PP_\theta$, and $\E_\theta\left[
\nabla E_\theta(\mathbf{\tilde x},\boldsymbol{H}) \,|\, \mathbf{\tilde x} \right]$ is the corresponding
conditional expectation when the hidden values follow
$\PP_\theta(\cdot\,|\,\mathbf{\tilde x})$ given the visible values 
$\mathbf{\tilde x}$.

A stochastic gradient method is then to sample visible values, $\boldsymbol{X}(\omega)$, 
according to $\PP_{\rm target}$ and update
$\theta$ according to the stochastic gradient:
\begin{align}
  g_\theta(\omega)
  & = \E_\theta\left[ \nabla E_\theta(\boldsymbol{X},\boldsymbol{H}) \right]
  -
  \E_\theta\left[ \nabla E_\theta(\boldsymbol{X}(\omega),\boldsymbol{H}) \,|\, \boldsymbol{X}(\omega) \right].
\end{align}

By taking into account the specific form of the energy, we find the
following specific update rule:
\begin{align}
  b_i
  & \leftarrow
  b_i + \eta \, \left( \E_\theta[ U_i \,|\, \boldsymbol{X}(\omega) ] -  \E_\theta[ U_i ]\right) \\
  w_{i,j}
  & \leftarrow
  w_{i,j} 
  + \eta \, \left( \E_\theta[ U_i\,U_j \,|\, \boldsymbol{X}(\omega) ] -  \E_\theta[ U_i\,U_j ] \right),
\end{align}
where each unit ($i$ or $j$) may be either visible or
hidden.  Specifically, let $M$ be the number of hidden units and $N$
be the number of visible units.  Then
$(i,j)\in[1,N+M-1]\times[i+1,N+M]$.  Here, $U_i$ denotes the value of
the $i$-th unit, which may be visible or hidden.  When the $i$-th unit
is visible, its expected value is simply $\E_\theta[ U_i \,|\,
  \boldsymbol{X}(\omega)] = X_i(\omega)$, and $\E_\theta[ U_i\,U_j \,|\,
  \boldsymbol{X}(\omega) ] = X_i(\omega) \, \E_\theta[ U_j \,|\,
  \boldsymbol{X}(\omega) ]$.  When both $i$ and $j$ are visible, we have
$\E_\theta[ U_i\,U_j \,|\, \boldsymbol{X}(\omega) ] = X_i(\omega)
X_j(\omega)$.  

Namely, we have
\begin{align}
  b_i
  & \leftarrow
  b_i + \eta \, \left( X_i(\omega) -  \E_\theta[ X_i ]\right)
\intertext{for a visible unit $i \in[1,N]$,}
  b_i
  & \leftarrow
  b_i + \eta \, \left( \E_\theta[ H_i \,|\, \boldsymbol{X}(\omega) ] -  \E_\theta[ H_i ]\right) 
\intertext{for a hidden unit $i\in[N+1,N+M]$,}
  w_{i,j}
  & \leftarrow
  w_{i,j} 
  + \eta \, \left( X_i(\omega)\,X_j(\omega) -  \E_\theta[ X_i\,X_j ] \right)
\intertext{for a pair of visible units $(i,j) \in [1,N-1]\times[i+1,N]$,}
  w_{i,j}
  & \leftarrow
  w_{i,j} 
  + \eta \, \left( X_i(\omega)\,\E_\theta[ H_j \,|\, \boldsymbol{X}(\omega) ] -  \E_\theta[ X_i\,H_j ] \right) 
\intertext{for a pair of a visible unit and a hidden unit $(i,j)\in[1,N]\times[N+1,N+M]$, and}
  w_{i,j}
  & \leftarrow
  w_{i,j} 
  + \eta \, \left( \E_\theta[ H_i\,H_j \,|\, \boldsymbol{X}(\omega) ] -  \E_\theta[ H_i\,H_j ] \right) 
\end{align}
for a pair of hidden units $(i,j)\in[N+1,N+M-1]\times[i+1,N+M]$.

\subsubsection{Hessian}

We now derive the Hessian of $f(\theta)$ when some of the units are 
hidden. From the gradient of $f(\theta)$ in \eqref{eq:BM:gradH-hidden-2}, 
we can write the partial derivatives as follows:
\begin{align}
  \frac{\partial}{\partial w_{i,j}} f(\theta)
  & = \sum_{\mathbf{\tilde x}} \PP_{\rm target}(\mathbf{\tilde x}) 
	\sum_{\mathbf{\tilde h}} \PP_\theta(\mathbf{\tilde h} \,|\, \mathbf{\tilde x}) \tilde u_i \, \tilde u_j
	- \sum_{\mathbf{\tilde u}} \PP_\theta(\mathbf{\tilde u}) \, \tilde u_i \, \tilde u_j
\label{eq:BM:partial1-hidden} \\
  \frac{\partial}{\partial w_{k,\ell}} \frac{\partial}{\partial w_{i,j}} f(\theta)
  & = \sum_{\mathbf{\tilde x}} \PP_{\rm target}(\mathbf{\tilde x}) 
	\sum_{\mathbf{\tilde h}} \frac{\partial}{\partial w_{k,\ell}} 
	\PP_\theta(\mathbf{\tilde h} \,|\, \mathbf{\tilde x}) \tilde u_i \, \tilde u_j
	- \sum_{\mathbf{\tilde u}} \frac{\partial}{\partial w_{k,\ell}} \PP_\theta(\mathbf{\tilde u}) \, \tilde u_i \, \tilde u_j,
\label{eq:BM:partial2-hidden}
\end{align}
where recall that $u_i\equiv x_i$ for $i\in[1,N]$ and $u_{N+i}\equiv h_i$ for $i\in[1,M]$.

When all of the units are visible, the first term in 
\eqref{eq:BM:partial-w} is the expectation with respect to the target 
distribution and is independent of $\theta$.  Now that some of the units 
are hidden, the corresponding first term in \eqref{eq:BM:partial1-hidden} 
depends on $\theta$.  Here, the first term is expectation, where the 
visible units follow the target distribution, and the hidden units 
follow the conditional distribution with respect to the Boltzmann 
machine with parameters $\theta$ given the visible values.

Notice that the second term in the right-hand side of 
\eqref{eq:BM:partial2-hidden} has the form equivalent to 
\eqref{eq:BM:partial2}.  Hence, we have
\begin{align}
\sum_{\mathbf{\tilde u}} \frac{\partial}{\partial w_{k,\ell}} \PP_\theta(\mathbf{\tilde u}) \, \tilde u_i \, \tilde u_j
= \Cov_\theta[ U_i\,U_j, U_k\,U_\ell ].
\end{align}

The first term in the right-hand side of \eqref{eq:BM:partial2-hidden} 
has a from similar to \eqref{eq:BM:partial2} with two differences.  The 
first difference is that it is an expectation with respect to $\PP_{\rm 
target}$.  The second difference is that the probability $\PP_\theta$ is 
conditional probability given the visible values.  We now show that 
there exists a Boltzmann distribution that has $\PP_\theta(\cdot \,|\, 
\mathbf{x})$ as its probability distribution (see also Figure~\ref{fig:BM:conditional}).
\begin{lemma}
Consider a Boltzmann machine having visible units and hidden units whose energy is given by \eqref{eq:BM:energy-hidden-2}.  The conditional probability distribution of the hidden units given the visible values $\mathbf{x}$ is given by the probability distribution of a Boltzmann machine with bias $\mathbf{b}(\mathbf{x})$ and weight $\mathbf{W}(\mathbf{x})$ that are defined as follows:
\begin{align}
\mathbf{b}(\mathbf{x})
 & \equiv \mathbf{b}^{\rm H} + (\mathbf{W}^{\rm VH})^\top \mathbf{x} 
\label{eq:BM:conditional-b}\\
\mathbf{W}(\mathbf{x})
 & \equiv \mathbf{W}^{\rm HH}.
\label{eq:BM:conditional-w}
\end{align}
\label{lemma:BM:conditional}
\end{lemma}
\begin{proof}
Recall the expression of the conditional probability distribution in \eqref{eq:BM:conditional-probability}:
\begin{align}
\PP_\theta(\mathbf{h} \,|\, \mathbf{x})
& = \frac{\exp\left( -E_\theta(\mathbf{x}, \mathbf{h}) \right)}
	{\displaystyle\sum_{\mathbf{\tilde h}} \exp\left( -E_\theta(\mathbf{x}, \mathbf{h}) \right)},
\label{eq:BM:conditional-probability-2}
\end{align}
where we rewrite the energy in \eqref{eq:BM:energy-hidden-2} as follows:
\begin{align}
E_\theta(\mathbf{x},\mathbf{h})
& = - \left((\mathbf{b}^{\rm V})^\top \mathbf{x} + \mathbf{x}^\top \mathbf{W}^{\rm VV} \mathbf{x}\right)
	- \left((\mathbf{b}^{\rm H})^\top + \mathbf{x}^\top \mathbf{W}^{\rm VH}\right) \mathbf{h}
	- \mathbf{h}^\top \mathbf{W}^{\rm HH} \, \mathbf{h}.
\end{align}
The first term in the right-hand side of the last expression is independent of $\mathbf{h}$ and canceled out between the numerator and the denominator in \eqref{eq:BM:conditional-probability-2}.  We thus consider the Boltzmann machine with parameters $\theta(\mathbf{x}) \equiv \left(\mathbf{b}(\mathbf{x}), \mathbf{W}(\mathbf{x})\right)$ as defined in \eqref{eq:BM:conditional-b}-\eqref{eq:BM:conditional-w}.  Then we have
\begin{align}
\PP_\theta(\mathbf{h} \,|\, \mathbf{x})
& = \frac{\exp\left( -E_{\theta(\mathbf{x})}(\mathbf{h}) \right)}
	{\displaystyle\sum_{\mathbf{\tilde h}} \exp\left( -E_{\theta(\mathbf{x})}(\mathbf{\tilde h}) \right)} \\
& = \PP_{\theta(\mathbf{x})}(\mathbf{h}),
\end{align}
which completes the proof of the lemma.
\end{proof}

\begin{figure}
\begin{minipage}{0.4\linewidth}
\centering
\includegraphics[width=\linewidth]{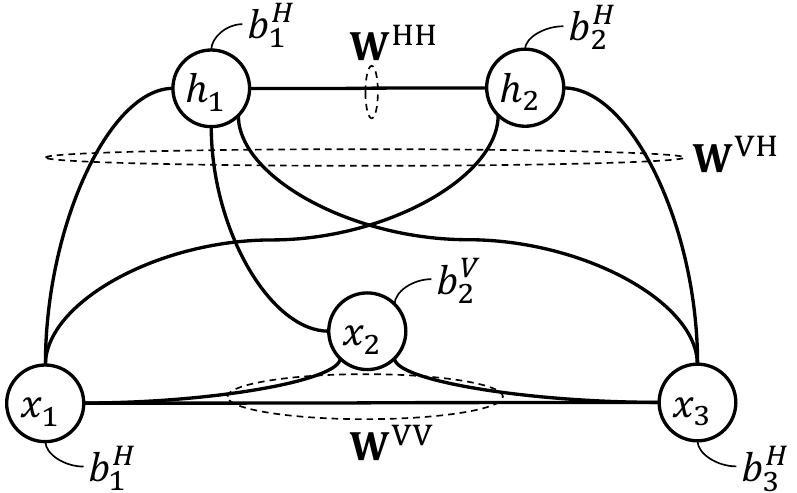}\\
(a)
\end{minipage}
\begin{minipage}{0.2\linewidth}
\ 
\end{minipage}
\begin{minipage}{0.4\linewidth}
\centering
\includegraphics[width=\linewidth]{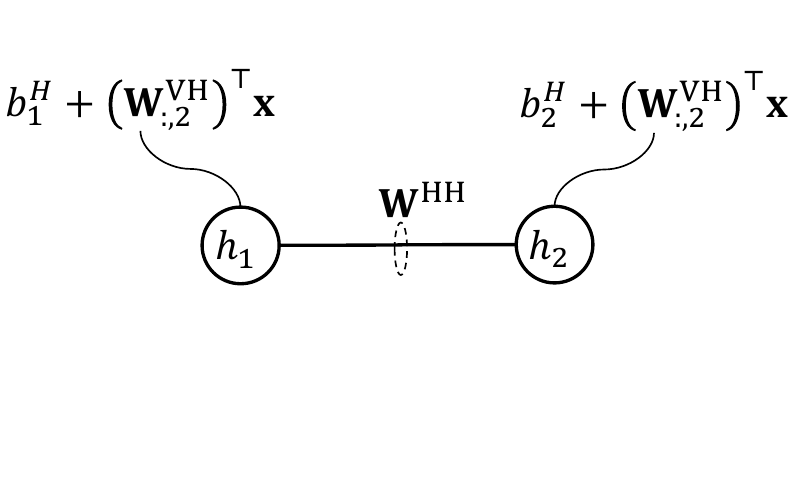}\\
(b)
\end{minipage}
\caption{The conditional distribution of $\mathbf{h}$ given $\mathbf{x}$ in (a) is equal to the distribution of $\mathbf{h}$ in (b), illustrating Lemma~\ref{lemma:BM:conditional}.}
\label{fig:BM:conditional}
\end{figure}

By \eqref{eq:BM:partial2-hidden} and Lemma~\ref{lemma:BM:conditional}, we can represent the second partial derivative as follows:
\begin{align}
  \frac{\partial}{\partial w_{k,\ell}} \frac{\partial}{\partial w_{i,j}} f(\theta)
  & = \E_{\rm target}\left[
	\Cov_{\theta}[U_i\,U_j, U_k\,U_\ell \,|\, \boldsymbol{X} ]
	\right] 
	- \Cov_{\theta}[U_i\,U_j, U_k\,U_\ell ],
\end{align}
where $\Cov_{\theta}[\cdot,\cdot \,|\, \boldsymbol{X}]$ denotes the 
conditional covariance with respect to 
$\PP_{\theta(\boldsymbol{X})}(\cdot)\equiv \PP_\theta(\cdot\,|\,\boldsymbol{X})$ 
given the visible values $\boldsymbol{X}$.

Therefore, the Hessian of $f(\theta)$ is
\begin{align}
  \nabla^2 f(\theta)
  & = \E_{\rm target}\left[
	\Cov_{\theta}[\boldsymbol{S} \,|\, \boldsymbol{X} ]
	\right] 
	- \Cov_{\theta}[\boldsymbol{S}],
\end{align}
where we define $\boldsymbol{S}$ as in \eqref{eq:BM:S}, and 
$\Cov_{\theta}[\cdot,\cdot \,|\, \boldsymbol{X}]$ denotes the conditional 
covariance matrix with respect to $\PP_{\theta(\boldsymbol{X})}$ given the visible values 
$\boldsymbol{X}$.  In general, the Hessian is not 
positive semidefinite, and $f(\theta)$ is not concave.  Hence, 
(stochastic) gradient based approaches do not necessarily find globally 
optimal parameters.

\subsubsection{Summary}

Consider a Boltzmann machine with parameters 
$\theta\equiv(\mathbf{b},\mathbf{W})$, where at least one of the units 
are hidden. The Boltzmann machine defines a probability distribution 
$\PP_\theta$ of the visible values $\mathbf{x}$ and the hidden values 
$\mathbf{h}$ by
\begin{align}
  \PP_\theta(\mathbf{x},\mathbf{h})
  & = \frac{\exp\left( -E_\theta(\mathbf{x},\mathbf{h}) \right)}
     {\displaystyle\sum_{\mathbf{\tilde x},\mathbf{\tilde h}} \exp\left( -E_\theta(\mathbf{\tilde x},\mathbf{\tilde h}) \right)},
\label{eq:BM:hidden:P}
\end{align}
where energy is given by
\begin{align}
  E_\theta(\mathbf{x},\mathbf{h})
  & = -\mathbf{b}^\top \, {\mathbf{x}\choose\mathbf{h}}
  - \left(\mathbf{x}^\top,\mathbf{h}^\top\right) \, \mathbf{W} \, {\mathbf{x}\choose\mathbf{h}}.
\end{align}

The marginal probability distribution of the visible values is
\begin{align}
  \PP_\theta(\mathbf{x})
  & = \sum_{\mathbf{\tilde h}} \PP_\theta(\mathbf{x},\mathbf{\tilde h}) \\
  & = \frac{\exp\left( -F_\theta(\mathbf{x}) \right)}
     {\displaystyle\sum_{\mathbf{\tilde x}} \exp\left( -F_\theta(\mathbf{\tilde x}) \right)},
\end{align}
where free energy is defined as follows:
\begin{align}
  F_\theta(\mathbf{x})
  & \equiv - \log \sum_{\mathbf{\tilde h}} \exp\left( -E_\theta(\mathbf{x},\mathbf{\tilde h}) \right).
\end{align}

The KL divergence from $\PP_\theta$ to $\PP_{\rm target}$ can be
minimized (or the log-likelihood of the target data having the
empirical distribution $\PP_{\rm target}$ can be maximized) by
maximizing
\begin{align}
  f(\theta)
  \equiv
  \E_{\rm target}\left[ \log \PP_\theta(\boldsymbol{X}) \right].
\end{align}
The gradient and the Hessian of $f(\theta)$ are given by
\begin{align}
  \nabla f(\theta)
  & = \E_{\rm target}\left[ \E_\theta[ \boldsymbol{S} \,|\, \boldsymbol{X} ] \right]
	- \E_\theta[\boldsymbol{S}] 
  \label{eq:BM:hidden-summary-grad}\\
  \nabla^2 f(\theta) 
  & = \E_{\rm target}\left[
	\Cov_{\theta}[\boldsymbol{S} \,|\, \boldsymbol{X} ]
	\right] 
	- \Cov_{\theta}[\boldsymbol{S}],
  \label{eq:BM:hidden-summary-laplacian}
\end{align}
where $\boldsymbol{S}$ denotes the vector of the random variables
representing the value of a unit ($U_i=X_i$ for $i\in[1,N]$ and $U_{N+i}=H_i$ for $i\in[1,M]$) or the product of the values of a
pair of units:
\begin{align}
  \boldsymbol{S} \equiv (U_1, \ldots, U_{N+M}, U_1\,U_2, \ldots, U_{N+M-1}\,U_{N+M}).
\end{align}

\section{Learning a discriminative model}
\label{sec:BM:discriminative}

In this section, we study the Boltzmann machine whose visible units are 
divided into input units and output units (see 
Figure~\ref{fig:BM:roles:inout}).  Such a Boltzmann machine is sometimes called a 
conditional Boltzmann machine \cite{Conditional}.  Let $N_{\rm in}$ be 
the number of input units, $N_{\rm out}$ be the number of output units, 
$N=N_{\rm in}+N_{\rm out}$ be the number of visible units, and $M$ be 
the number of hidden units.

Such a Boltzmann machine can be used 
to model a conditional probability distribution of the output values 
$\mathbf{y}$ given the input values $\mathbf{x}$.  Let $\PP_{\theta}(\mathbf{y}\,|\,\mathbf{x})$ denote the 
conditional probability of $\mathbf{y}$ given $\mathbf{x}$. 
 When the Boltzmann machine has hidden units, we also write
\begin{align}
\PP_\theta(\mathbf{y} \,|\, \mathbf{x})
& = \sum_{\mathbf{\tilde h}} 
 \PP_\theta(\mathbf{y},\mathbf{\tilde h} \,|\, \mathbf{x})
 \label{eq:BM:Pyx}\\
& = \frac{
\displaystyle\sum_{\mathbf{\tilde h}} \PP_\theta(\mathbf{x},\mathbf{y},\mathbf{\tilde h})
}{
\displaystyle\sum_{\mathbf{\tilde y},\mathbf{\tilde h}} \PP_\theta(\mathbf{x},\mathbf{y},\mathbf{\tilde h})
 }.
\end{align}

\subsection{Objective function}

A training dataset for a discriminative model is a set of the pairs of input and output:
\begin{align}
{\cal D} = \{ (\mathbf{x}^{(d)}, \mathbf{y}^{(d)}) \}_{d=1}^D.
\end{align}
When this training dataset is given, a natural objective function corresponding to the log-likelihood \eqref{eq:BM:log-likelihood} for the generative model is
\begin{align}
f(\theta)
& = \frac{1}{D} \sum_{(\mathbf{x},\mathbf{y})\in{\cal D}} 
	\log \PP_\theta(\mathbf{y}\,|\,\mathbf{x}) 
\label{eq:BM:discriminative-f}\\
& = \frac{1}{D} 
	\log \prod_{(\mathbf{x},\mathbf{y})\in{\cal D}} \PP_\theta(\mathbf{y}\,|\,\mathbf{x}).
\end{align}

This objective function may also be related to a KL divergence.  Let 
$\PP_{\rm target}(\cdot\,|\,\cdot)$ be the target conditional 
distribution, which we seek to model with $\PP_\theta(\cdot\,|\,\cdot)$. 
 Consider the following expected KL divergence from 
$\PP_\theta(\cdot\,|\,\boldsymbol{X})$ to $\PP_{\rm 
target}(\cdot\,|\,\boldsymbol{X})$, where $\boldsymbol{X}$ is the random 
variable denoting the input values that are distributed according to 
$\PP_{\rm target}$:
\begin{align}
\lefteqn{\E_{\rm target}\left[
{\rm KL}(\PP_{\rm target}(\cdot\,|\,\boldsymbol{X}) \,||\, \PP_\theta(\cdot\,|\,\boldsymbol{X}))
\right]} \notag\\
& = \sum_{\mathbf{\tilde x}} \PP_{\rm target}(\mathbf{\tilde x})
	\sum_{\mathbf{\tilde y}} 
	\PP_{\rm target}(\mathbf{\tilde y}\,|\,\mathbf{\tilde x}) \,
	\log \frac{\PP_{\rm target}(\mathbf{\tilde y}\,|\,\mathbf{\tilde x})}
		{\PP_\theta(\mathbf{\tilde y}\,|\,\mathbf{\tilde x})} \\
& = \sum_{\mathbf{\tilde x},\mathbf{\tilde y}}
	\PP_{\rm target}(\mathbf{\tilde x},\mathbf{\tilde y}) \,
	\log \PP_{\rm target}(\mathbf{\tilde y}\,|\,\mathbf{\tilde x})
  - \sum_{\mathbf{\tilde x},\mathbf{\tilde y}} 
	\PP_{\rm target}(\mathbf{\tilde x},\mathbf{\tilde y}) \,
	\log \PP_\theta(\mathbf{\tilde y}\,|\,\mathbf{\tilde x}).
\end{align}

The first term of the last expression is independent of $\theta$.  To minimize the expected KL divergence, it thus suffices to maximize the negation of the second term:
\begin{align}
f(\theta)
& = \sum_{\mathbf{\tilde x},\mathbf{\tilde y}} 
	\PP_{\rm target}(\mathbf{\tilde x},\mathbf{\tilde y}) \,
	\log \PP_\theta(\mathbf{\tilde y}\,|\,\mathbf{\tilde x})
\label{eq:BM:discriminative-f-dist}\\
& = \sum_{\mathbf{\tilde x},\mathbf{\tilde y}} 
	\PP_{\rm target}(\mathbf{\tilde x},\mathbf{\tilde y}) \,
	\log \frac{\PP_\theta(\mathbf{\tilde x}, \mathbf{\tilde y})}{\PP_\theta(\mathbf{\tilde x})} \\
& = \sum_{\mathbf{\tilde x},\mathbf{\tilde y}} 
	\PP_{\rm target}(\mathbf{\tilde x},\mathbf{\tilde y}) \,
	\log \PP_\theta(\mathbf{\tilde x}, \mathbf{\tilde y})
  - \sum_{\mathbf{\tilde x}} 
	\PP_{\rm target}(\mathbf{\tilde x}) \,
	\log \PP_\theta(\mathbf{\tilde x}) \\
& = \E_{\rm target}\left[
	\log \PP_\theta(\boldsymbol{X}, \mathbf{Y})
	\right]
  - \E_{\rm target}\left[
	\log \PP_\theta(\boldsymbol{X})
	\right]
\label{eq:BM:discriminative-f-in-E}
\end{align}
When $\PP_{\rm target}$ is the empirical distribution of the pairs of input and output in the training dataset, \eqref{eq:BM:discriminative-f-dist} is reduced to \eqref{eq:BM:discriminative-f}.

\subsection{Gradient, stochastic gradient, Hessian}

Observe that the first term of \eqref{eq:BM:discriminative-f-in-E} is 
equivalent to the objective function of the generative model where both 
of input and output are visible ({\it i.e.}, input and output in 
\eqref{eq:BM:discriminative-f-in-E} should be regarded as visible in 
\eqref{eq:BM:hidden-summary-grad}).  The second term 
is analogous to the first term, but now only the input should be 
regarded as visible ({\it i.e.}, output and hidden in 
\eqref{eq:BM:discriminative-f-in-E} should be regarded as hidden in 
\eqref{eq:BM:hidden-summary-grad}).

The gradient thus follows from \eqref{eq:BM:hidden-summary-grad}:
\begin{align}
\nabla f(\theta)
& = \E_{\rm target}\left[ \E_{\theta}[ \boldsymbol{S} \,|\, \boldsymbol{X}, \mathbf{Y} ] \right]
	- \E_\theta[ \boldsymbol{S} ]
  - \left(
	\E_{\rm target}\left[ \E_{\theta}[ \boldsymbol{S} \,|\, \boldsymbol{X} ] \right]
	- \E_\theta[ \boldsymbol{S} ]
    \right)\\
& = \E_{\rm target}\left[ \E_{\theta}[ \boldsymbol{S} \,|\, \boldsymbol{X}, \mathbf{Y} ] 
  	- \E_{\theta}[ \boldsymbol{S} \,|\, \boldsymbol{X} ] \right],
\label{eq:BM:discriminative-grad}
\end{align}
where $\boldsymbol{S}$ denotes the vector of random variables representing 
the value of a unit or the product of the values of a pair of units:
\begin{align}
\boldsymbol{S}
\equiv
(U_1, \dots, U_{N+M}, U_1\,U_2, \ldots, U_{N+M-1}\,U_{N+M}),
\end{align}
where $U_i \equiv X_i$ for $i\in[1,N_{\rm in}]$, $U_{N_{\rm in}+i} 
\equiv Y_i$ for $i\in[1,N_{\rm out}]$, and $U_{N+i} \equiv H_i$ for 
$i\in[1,M]$.

Because \eqref{eq:BM:discriminative-grad} is expressed as an expectation 
with respect to $\PP_{\rm target}$, it directly gives the following 
stochastic gradient:
\begin{align}
g_\theta(\omega)
& = \E_{\theta}[ \boldsymbol{S} \,|\, \boldsymbol{X}(\omega), \mathbf{Y}(\omega) ] 
  	- \E_{\theta}[ \boldsymbol{S} \,|\, \boldsymbol{X}(\omega) ],
\end{align}
where $(\boldsymbol{X}(\omega), \mathbf{Y}(\omega))$ is sampled according to $\PP_{\rm target}$.

Specifically, we have the following learning rule of a stochastic gradient method:
\begin{align}
  b_i
  & \leftarrow
  b_i + \eta \, \left( Y_i(\omega) -  \E_\theta[ Y_i \,|\, \boldsymbol{X}(\omega) ]\right)
\intertext{for an output unit $i \in [N_{\rm in}+1,N]$,}
  b_i
  & \leftarrow
  b_i + \eta \, \left( \E_\theta[H_i \,|\, \boldsymbol{X}(\omega), \mathbf{Y}(\omega)]
	-  \E_\theta[ H_i \,|\, \boldsymbol{X}(\omega) ]\right)
\intertext{for a hidden unit $i \in [N+1,N+M]$,}
  w_{i,j}
  & \leftarrow
  w_{i,j} 
  + \eta \, \left( X_i(\omega) \, Y_j(\omega) - X_i(\omega)\,\E_\theta[Y_j\,|\,\boldsymbol{X}(\omega)] \right)
\intertext{for a pair of an input unit and an output unit $(i,j)\in[1,N_{\rm in}]\times[N_{\rm in}+1,N]$,}
  w_{i,j}
  & \leftarrow
  w_{i,j} 
  + \eta \, \left( X_i(\omega) \, \E_\theta\left[H_j\,|\,\boldsymbol{X}(\omega),\mathbf{Y}(\omega)\right] - X_i(\omega)\,\E_\theta[H_j\,|\,\boldsymbol{X}(\omega)] \right)
\intertext{for a pair of an input unit and a hidden unit $(i,j)\in[1,N_{\rm in}]\times[N+1,N+M]$,}
  w_{i,j}
  & \leftarrow
  w_{i,j} + \eta \, \left( Y_i(\omega) \, Y_j(\omega) - \E_\theta[Y_i\,Y_j\,|\,\boldsymbol{X}(\omega)] \right)
\intertext{for a pair of output units $(i,j)\in[N_{\rm in}+1,N-1]\times[i+1,N]$,}
  w_{i,j}
  & \leftarrow
  w_{i,j} 
  + \eta \, \left( Y_i(\omega) \, \E_\theta\left[H_j\,|\,\boldsymbol{X}(\omega),\mathbf{Y}(\omega)\right] - \E_\theta[Y_i\,H_j\,|\,\boldsymbol{X}(\omega)] \right)
\intertext{for a pair of an output unit and a hidden unit $(i,j)\in[N_{\rm in}+1,N]\times[N+1,N+M]$, and}
  w_{i,j}
  & \leftarrow
  w_{i,j} 
  + \eta \, \left( \E_\theta\left[H_i\,H_j\,|\,\boldsymbol{X}(\omega),\mathbf{Y}(\omega)\right] - \E_\theta[H_i\,H_j\,|\,\boldsymbol{X}(\omega)] \right)
\end{align}
for a pair of hidden units $(i,j)\in[N+1,N+M-1]\times[i+1,N+M]$.

We also obtain trivial learning rules:
\begin{align}
  b_i & \leftarrow b_i
\intertext{for an input unit $i \in[1,N_{\rm in}]$ and}
  w_{i,j} & \leftarrow w_{i,j}
\end{align}
for a pair of input units $(i,j)\in[1,N_{\rm in}-1]\times[i+1,N_{\rm 
in}]$.  One can easily see that these parameters are redundant and do 
not play any role in $\PP_\theta(\mathbf{y}\,|\,\mathbf{x})$.

The Hessian of the objective of discriminative learning follows from 
\eqref{eq:BM:hidden-summary-laplacian} and \eqref{eq:BM:discriminative-f-in-E}:
\begin{align}
  \nabla^2 f(\theta) 
  & = \E_{\rm target}\left[
	\Cov_\theta[\boldsymbol{S} \,|\, \boldsymbol{X}, \mathbf{Y} ]
	\right] 
	- \Cov_{\theta}[\boldsymbol{S}]
  -\left( \E_{\rm target}\left[
	\Cov_{\theta}[\boldsymbol{S} \,|\, \boldsymbol{X} ]
	\right] 
	- \Cov_{\theta}[\boldsymbol{S}]
	\right) \\
  & = \E_{\rm target}\left[
	\Cov_{\theta}[\boldsymbol{S} \,|\, \boldsymbol{X}, \mathbf{Y} ]
	- \Cov_{\theta}[\boldsymbol{S} \,|\, \boldsymbol{X} ]
	\right],
\end{align}
where $\Cov_\theta[\cdot\,|\,\boldsymbol{X},\mathbf{Y}]$ denotes the 
conditional covariance matrix with respect to the conditional 
distribution of the hidden values 
$\PP_\theta(\cdot\,|\,\boldsymbol{X},\mathbf{Y})$ given 
$(\boldsymbol{X},\mathbf{Y})$, and $\Cov_\theta[\cdot\,|\,\boldsymbol{X}]$ 
denotes the conditional covariance matrix with respect to the 
conditional distribution of the output and hidden values 
$\PP_\theta(\cdot\,|\,\boldsymbol{X})$ given $\boldsymbol{X}$.

\subsubsection{Summary}

Consider a Boltzmann machine with parameters 
$\theta=(\mathbf{b},\mathbf{W})$, where the units can be classified into 
input, output, and hidden. The Boltzmann machine defines a conditional 
probability distribution of the output values $\mathbf{y}$ given the 
input values $\mathbf{x}$:
\begin{align}
  \PP_\theta(\mathbf{y}\,|\,\mathbf{x})
& = \frac{
\displaystyle\sum_{\mathbf{\tilde h}} \PP_\theta(\mathbf{x},\mathbf{y},\mathbf{\tilde h})
}{
\displaystyle\sum_{\mathbf{\tilde y},\mathbf{\tilde h}} \PP_\theta(\mathbf{x},\mathbf{\tilde y},\mathbf{\tilde h})
},
\end{align}
where $\PP_\theta(\mathbf{x},\mathbf{y},\mathbf{h})$ is the 
probability distribution of the Boltzmann machine where 
$(\mathbf{x},\mathbf{y})$ is visible, and $\mathbf{h}$ is hidden.

The expected KL divergence from $\PP_\theta(\cdot\,|\,\boldsymbol{X})$ to 
$\PP_{\rm target}(\cdot\,|\,\boldsymbol{X})$, where the expectation is with 
respect to the target distribution of the input values $\boldsymbol{X}$, can 
be minimized (or the sum of the conditional log-likelihood of the output 
values given the input values, where the input and output follow the 
empirical distribution $\PP_{\rm target}$, can be maximized) by 
maximizing
\begin{align}
  f(\theta)
  \equiv
  \E_{\rm target}\left[ \log \PP_\theta(\boldsymbol{X},\mathbf{Y}) \right]
  -
  \E_{\rm target}\left[ \log \PP_\theta(\boldsymbol{X}) \right].
\end{align}
The gradient and the Hessian of $f(\theta)$ are given by
\begin{align}
  \nabla f(\theta)
  & = \E_{\rm target}\left[ \E_\theta[ \boldsymbol{S} \,|\, \boldsymbol{X},\mathbf{Y} ] \right]
	- \E_{\rm target}[\E_\theta[\boldsymbol{S} \,|\, \boldsymbol{X}]] 
  \label{eq:BM:disc:grad}\\
  \nabla^2 f(\theta) 
  & = \E_{\rm target}\left[
	\Cov_{\theta}[\boldsymbol{S} \,|\, \boldsymbol{X}, \mathbf{Y} ]
	\right] 
  - \E_{\rm target}\left[
	\Cov_{\theta}[\boldsymbol{S} \,|\, \boldsymbol{X} ]
	\right], 
\end{align}
where $\boldsymbol{S}$ denotes the vector of the random variables
representing the value of a unit ($U_i=X_i$ for $i\in[1,N_{\rm in}]$, $U_{N_{\rm in}+i}=Y_i$ for $i\in[1,N_{\rm out}]$, and $U_{N+i}=H_i$ for $i\in[1,M]$) or the product of the values of a
pair of units:
\begin{align}
  \boldsymbol{S} \equiv (U_1, \ldots, U_{N+M}, U_1\,U_2, \ldots, U_{N+M-1}\,U_{N+M}).
\end{align}

\subsection{Simplest case with relation to logistic regression}

Here, we study the simplest but non-trivial case of a discriminative 
model of a Boltzmann machine.  Specifically, we assume that the 
Boltzmann machine has no hidden units, and there are no weight between 
output units.  As we have discussed in 
Section~\ref{sec:BM:discriminative}, the bias associated with the 
input unit and the weight between input units are redundant and do 
not play any role.  Therefore, it suffices to consider the bias for 
output units and the weight between input units and output units (see 
Figure~\ref{fig:BM:simple}).  Let $\mathbf{b}^{\rm O}$ be the bias 
associated with output units and $\mathbf{W}^{\rm IO}$ be the weight 
matrix where the $(i,j)$ element denotes the weight between the $i$-th 
input unit and the $j$-th output unit.

\begin{figure}
\centering
\includegraphics[width=0.4\linewidth]{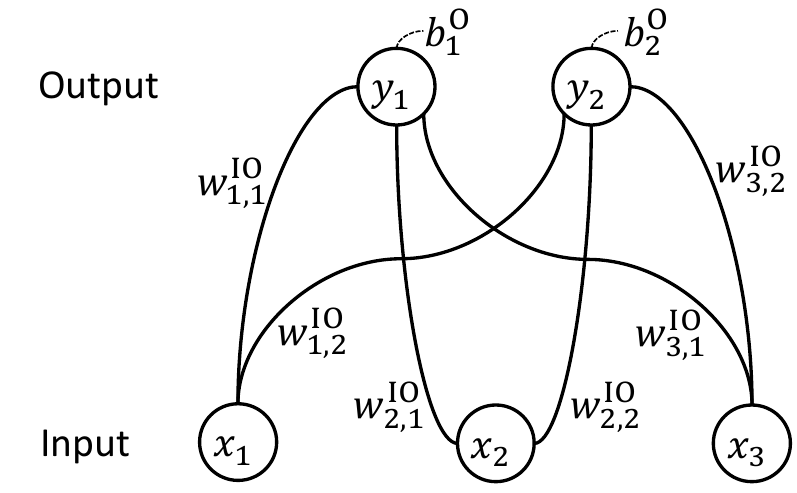}
\caption{A discriminative model with a Boltzmann machine with no hidden units and no weight between output units.}
\label{fig:BM:simple}
\end{figure}

\subsubsection{Conditional probability}

We first apply Lemma \ref{lemma:BM:conditional}, where the input units 
correspond to the visible units in the lemma, and the output units 
correspond to the hidden units in the lemma.  Then we can see that the 
probability distribution of output values given input values 
$\mathbf{x}$ is given by a Boltzmann machine with the following bias and no weight:
\begin{align}
\mathbf{b}(\mathbf{x})
 & \equiv \mathbf{b}^{\rm O} + (\mathbf{W}^{\rm IO})^\top \mathbf{x}.
\end{align}

The conditional probability of output values $\mathbf{y}$ given input values $\mathbf{x}$ is thus given by
\begin{align}
\PP_\theta(\mathbf{y}\,|\,\mathbf{x})
& = \frac{
\exp\left( (\mathbf{b}(\mathbf{x}))^\top \mathbf{y} \right)
}{
\displaystyle\sum_{\mathbf{\tilde y}} 
\exp\left( (\mathbf{b}(\mathbf{x}))^\top \mathbf{\tilde y} \right)
}.
\label{eq:BM:simple:conditional}
\end{align}

In this case, the partition function, which consists of $2^{N_{\rm 
out}}$ terms, can be computed in $O(N_{\rm out})$ time:
\begin{lemma}
The partition function (the denominator) in the right-hand side of 
\eqref{eq:BM:simple:conditional} can be written as follows:
\begin{align}
\sum_{\mathbf{\tilde y}} \exp\left( (\mathbf{b}(\mathbf{x}))^\top \mathbf{\tilde y} \right)
& = \prod_{i=1}^{N_{\rm out}} \left( 1 + \exp(b_i(\mathbf{x})) \right),
\end{align}
where the summation with respect to $\mathbf{\tilde y}$ is over all of the possible binary patterns of length $N_{\rm out}$.
\label{lemma:BM:partition}
\end{lemma}
\begin{proof}
\begin{align}
\sum_{\mathbf{\tilde y}} \exp\left( (\mathbf{b}(\mathbf{x}))^\top \mathbf{\tilde y} \right)
& = \sum_{\mathbf{\tilde y}} \exp\left( \sum_{i=1}^{N_{\rm out}} b_i(\mathbf{x}) \, \tilde y_i \right) \\
& = \sum_{\mathbf{\tilde y}} \prod_{i=1}^{N_{\rm out}} \exp\left( b_i(\mathbf{x}) \, \tilde y_i \right) \\
& = \prod_{i=1}^{N_{\rm out}} \left( 1 + \exp(b_i(\mathbf{x})) \right).
\end{align}
\end{proof}

Now, it can be easily shown that the output values are conditionally 
independent of each other given the input values:
\begin{corollary}
The conditional probability of output values \eqref{eq:BM:simple:conditional} can be written as 
the product of the conditional probabilities of each output value:
\begin{align}
\PP_\theta(\mathbf{y}\,|\,\mathbf{x})
& = \prod_{i=1}^{N_{\rm out}} \PP_\theta(y_i\,|\,\mathbf{x}),
\intertext{where}
\PP_\theta(y_i\,|\,\mathbf{x})
& = \frac{
\exp\left( b_i(\mathbf{x}) \, y_i \right)
}{
1 + \exp\left( b_i(\mathbf{x}) \right)
}.
\label{eq:BM:Pyi}
\end{align}
\label{corollary:BM:independence}
\end{corollary}
\begin{proof}
By Lemma~\ref{lemma:BM:partition}, we have
\begin{align}
\PP_\theta(\mathbf{y}\,|\,\mathbf{x})
& = \frac{\exp\left(\displaystyle\sum_{i=1}^{N_{\rm out}} b_i(\mathbf{x})\,y_i \right)}
	{\displaystyle\prod_{i=1}^{N_{\rm out}} \left( 1 + \exp(b_i(\mathbf{x})) \right)} \\
& = \prod_{i=1}^{N_{\rm out}} \PP_\theta(y_i\,|\,\mathbf{x}).
\end{align}
\end{proof}

\subsubsection{Relation to logistic regression}

By \eqref{eq:BM:Pyi}, the probability that the $i$-th output unit takes 
the value 1 is\begin{align}\PP_\theta(Y_i=1\,|\,\mathbf{x})& = 
\frac{1}{1 + \exp\left( - b_i(\mathbf{x}) \right)} \\& = \frac{1}{1 + 
\exp\left( - (b_i^{\rm O} + \mathbf{w}_i^\top \mathbf{x}) 
\right)},\end{align} where $\mathbf{w}_i$ is the $i$-th column of 
$\mathbf{W}^{\rm IO}$. The last expression has the form of logistic 
regression, where the explanatory variables $\mathbf{x}$ are binary.

Due to the conditional independence shown in 
Corollary~\ref{corollary:BM:independence}, we can say that the Boltzmann 
machine shown in Figure~\ref{fig:BM:simple} consists of $N$ independent 
models of logistic regression that have common explanatory variables.

\section{Evaluating expectation with respect to a model distribution}
\label{sec:BM:evaluate}

When we train a Boltzmann machine with a stochastic gradient method, we 
need to evaluate expected values with respect to the distribution 
defined by the Boltzmann machine.  Such expected values appear for 
example as $\E_\theta[\boldsymbol{S}]$ in \eqref{eq:BM:gradH-visible} and 
\eqref{eq:BM:hidden-summary-grad} or as 
$\E_\theta[\boldsymbol{S}\mid\boldsymbol{X}]$ in \eqref{eq:BM:disc:grad}.  
In general, exact expressions for these expectations are unavailable in closed forms.

\subsection{Gibbs sampler}

A practical approach that can be used to estimate those expectations is 
Markov Chain Monte Carlo in general and Glauber dynamics 
\cite{MarkovChainMixingTime} or Gibbs samplers \cite{GibbsSampler} in 
particular.  For example, we can sample $K$ patterns according to the 
distribution given by a Boltzmann machine via a Gibbs sampler shown in 
Algorithm~\ref{alg:BM:Gibbs}.  The expected values can then be estimated 
using the $K$ samples.  In practice, we often ignore initial samples and 
consider only every $n$th sample for a sufficiently large $n$.

\begin{algorithm}[t]
\begin{algorithmic}[1]
\State {\bf Input} Parameters of the Boltzmann machine and $K$
\State $\mathbf{x}^{(0)} \leftarrow $ Initialize the values of the $N$ units
\For{$k = 1, \ldots, K$}
 \For{$i = 1, \ldots, N$}
   \State Choose $x_i^{(k)}$ according to its conditional 
	distribution given $x_j^{(k-1)}$ for $j\neq i$
 \EndFor
\EndFor
\State {\bf Return} $K$ samples: $\mathbf{x}^{(k)}$ for $k=1,\ldots,K$.
\end{algorithmic}
\caption{A Gibbs sampler for a Boltzmann machine with $N$ units.}
\label{alg:BM:Gibbs}
\end{algorithm}

In Step 5 of Algorithm~\ref{alg:BM:Gibbs}, the 
conditional distribution of $x_i^{(k)}$ given $x_j^{(k-1)}$ for $j\neq i$ is defined by
\begin{align}
\PP_\theta(x_i^{(k)}\mid\mathbf{x}^{(k-1)})
& = \frac{\exp\big(-E_\theta(x_i^{(k)}\mid \mathbf{x}^{(k-1)})\big)}
   {\displaystyle\sum_{\tilde x_i^{(k)}\in\{0,1\}} \exp\big(-E_\theta(\tilde x_i^{(k)}\mid \mathbf{x}^{(k-1)})\big)},
\intertext{for $x_i^{(k)}\in\{0,1\}$, where}
E_\theta(x_i^{(k)}\mid \mathbf{x}^{(k-1)})
& \equiv - b_i \, x_i^{(k)} - \sum_{j\neq i} x_i^{(k)} \, w_{i,j} \, x_j^{(k-1)}.
\end{align}

\subsection{Contrastive divergence}
\label{sec:BM:CD}

Another approach to deal with the computational intractability of 
evaluation the expectations is to avoid it.  Namely, we modify our objective function.

Recall, from \eqref{eq:BM:KL} and \eqref{eq:BM:log-likelihood}, that our 
objective function has been the KL divergence from $\PP_\theta$ to 
$\PP_{\rm target}$ (or equivalently the log likelihood of data with 
respect to $\PP_\theta$).  The gradient of the KL divergence with respect 
to $\theta$ involves the computationally intractable term of expectation 
with respect to $\PP_\theta$.

Here, we study an alternative objective function of Contrastive 
divergence \cite{PoE,CD}.  Consider a Gibbs sampler 
(Algorithm~\ref{alg:BM:Gibbs}) that initializes the values by sampling 
from $\PP_{\rm target}$ (or uniformly at random from the dataset 
$\mathcal{D}$).  Because this is the distribution at the beginning of the 
Gibbs sampler, we write $\PP_0 \equiv \PP_{\rm target}$.  The 
distribution of the patterns sampled by the Gibbs sampler at step $k$ is 
referred to as $\PP_k^\theta$. Because $\PP_k^\theta \to \PP_\theta$ as $k\to\infty$, 
we also write $\PP_\infty^\theta\equiv\PP_\theta$.

The KL divergence in \eqref{eq:BM:KL} can now be written as
\begin{align}
{\rm KL}(\PP_0 \,||\, \PP_\infty^\theta)
& =   \sum_{\mathbf{\tilde x}}
  \PP_0(\mathbf{\tilde x})
  \log \PP_0(\mathbf{\tilde x})
  - \sum_{\mathbf{\tilde x}}
  \PP_0(\mathbf{\tilde x})
  \log \PP_\infty^\theta(\mathbf{\tilde x}),\\
\intertext{where}
\PP_\infty^\theta(\mathbf{x})
& = \frac{\exp(-E_\theta(\mathbf{x}))}
	{\displaystyle\sum_\mathbf{\tilde x} \exp(-E_\theta(\mathbf{\tilde x}))}
\end{align}
and its gradient (recall \eqref{eq:BM:grad-log}) as 
\begin{align}
\nabla_\theta {\rm KL}(\PP_0 \,||\, \PP_\infty^\theta)
& = \sum_\mathbf{\tilde x} \PP_0(\mathbf{\tilde x}) \,
	\nabla_\theta E_\theta(\mathbf{\tilde x})
- \sum_\mathbf{\tilde x} \PP_\infty^\theta(\mathbf{\tilde x}) \,
 \nabla_\theta E_\theta(\mathbf{\hat x}).
\label{eq:BM:KL0inf}
\end{align}
The first term of the right-hand side of \eqref{eq:BM:KL0inf} is the 
expectation with respect to the target distribution $\PP_0$, which is 
readily computable.  The second term is the expectation with respect to 
the model (with parameter $\theta$), which is in general computationally 
intractable for a Boltzmann machine.

To cancel out this computationally intractable second term, consider the 
following contrastive divergence:
\begin{align}
{\rm CD}_1(\theta) \equiv {\rm KL}(\PP_0 \,||\, \PP_\infty^\theta) - {\rm KL}(\PP_1^\theta \,||\, \PP_\infty^\theta).
\label{eq:BM:CD1}
\end{align}
Here, we quote an intuitive motivation of the contrastive divergence 
from \cite{PoE}, using our notations (shown within $[\cdot]$):
\begin{quote}
\it
The intuitive motivation for using this ``contrastive divergence'' is 
that we would like the Markov chain that is implemented by Gibbs 
sampling to leave the initial distribution $[\,\PP_0\,]$ over the visible 
variables unaltered. Instead of running the chain to equilibrium and 
comparing the initial and final derivatives we can simply run the chain 
for one full step and then update the parameters to reduce the tendency 
of the chain to wander away from the initial distribution on the first 
step. Because $[\,\PP_1^\theta\,]$ is one step closer to the equilibrium 
distribution than $[\,\PP_0\,]$, we are guaranteed that 
$[\,{\rm KL}(\PP_0 \,||\, \PP_\infty^\theta)\,]$ exceeds $[\,{\rm 
KL}(\PP_1^\theta \,||\, \PP_\infty^\theta)\,]$ unless $[\,\PP_0\,]$ equals $[\,\PP_1^\theta\,]$, 
so the contrastive divergence can never be negative. Also, for Markov 
chains in which all transitions have non-zero probability, $[\,\PP_0 = 
\PP_1^\theta\,]$ implies $[\,\PP_0 = \PP_1^\theta\,]$ so the contrastive divergence 
can only be zero if the model is perfect.
\end{quote}

The gradient of the second term of \eqref{eq:BM:CD1} can be derived as follows:
\begin{align}
{\rm KL}(\PP_1^\theta \,||\, \PP_\infty^\theta)
& =   \sum_{\mathbf{\tilde x}}
  \PP_1^\theta(\mathbf{\tilde x})
  \log \frac{\PP_1^\theta(\mathbf{\tilde x})}{\PP_\infty^\theta(\mathbf{\tilde x})} \\
\nabla_\theta {\rm KL}(\PP_1^\theta \,||\, \PP_\infty^\theta)
& =   \sum_{\mathbf{\tilde x}}
  \nabla_\theta \PP_1^\theta(\mathbf{\tilde x})
  \log \frac{\PP_1^\theta(\mathbf{\tilde x})}{\PP_\infty^\theta(\mathbf{\tilde x})}
+   \sum_{\mathbf{\tilde x}}
  \PP_1^\theta(\mathbf{\tilde x}) \,
	\nabla_\theta \log \PP_1^\theta(\mathbf{\tilde x})
  - \sum_{\mathbf{\tilde x}} \,
  \PP_1^\theta(\mathbf{\tilde x}) \,
	\nabla_\theta \log \PP_\infty^\theta(\mathbf{\tilde x}) \\
& =   \sum_{\mathbf{\tilde x}}
  \nabla_\theta \PP_1^\theta(\mathbf{\tilde x})
  \log \frac{\PP_1^\theta(\mathbf{\tilde x})}{\PP_\infty^\theta(\mathbf{\tilde x})}
+   \sum_{\mathbf{\tilde x}}
  \nabla_\theta \PP_1^\theta(\mathbf{\tilde x}) 
  - \sum_{\mathbf{\tilde x}} \,
  \PP_1^\theta(\mathbf{\tilde x}) \,
	\nabla_\theta \log \PP_\infty^\theta(\mathbf{\tilde x}) \\
& =   \sum_{\mathbf{\tilde x}}
  \nabla_\theta \PP_1^\theta(\mathbf{\tilde x})
\bigg(
	1
	+ \log \frac{\PP_1^\theta(\mathbf{\tilde x})}{\PP_\infty^\theta(\mathbf{\tilde x})}
\bigg)
 + \sum_{\mathbf{\tilde x}} \,
	\PP_1^\theta(\mathbf{\tilde x}) \,
	\nabla_\theta \E_\theta(\mathbf{\tilde x})
  - \sum_{\mathbf{\tilde x}} \,
	\PP_\infty^\theta(\mathbf{\tilde x})
	\nabla_\theta \E_\theta(\mathbf{\tilde x}),
\label{eq:BM:KL1inf}
\end{align}
where the last equality follows analogously to \eqref{eq:BM:grad-log}.

From \eqref{eq:BM:KL0inf}-\eqref{eq:BM:CD1} and \eqref{eq:BM:KL1inf}, we obtain
\begin{align}
\nabla_\theta {\rm CD}_1(\theta)
& = \sum_\mathbf{\tilde x} \PP_0(\mathbf{\tilde x}) \,
	\nabla_\theta E_\theta(\mathbf{\tilde x})
	- \sum_\mathbf{\tilde x} \PP_1^\theta(\mathbf{\tilde x}) \,
	\nabla_\theta E_\theta(\mathbf{\tilde x})
  - \varepsilon(\theta),
\label{eq:BM:gradCD}
\intertext{where}
\varepsilon(\theta)
& \equiv \sum_{\mathbf{\tilde x}}
  \nabla_\theta \PP_1^\theta(\mathbf{\tilde x})
\bigg(
	1
	+ \log \frac{\PP_1^\theta(\mathbf{\tilde x})}{\PP_\infty^\theta(\mathbf{\tilde x})}
\bigg).
\end{align}
Hinton has empirically shown that $\varepsilon(\theta)$ is small and recommended 
the approximation of $\varepsilon(\theta)\approx 0$ in \cite{PoE}.  
The first term of the right-hand side of \eqref{eq:BM:gradCD} is 
an expectation with respect to $\PP_0$ and can be readily evaluated.  
The second term is an expectation with respect to $\PP_1^\theta$ 
and can be estimated with the samples from the Gibbs sampler in one step.

In \cite{PoE}, $\varepsilon(\theta)$ is represented in the following equivalent form:
\begin{align}
\varepsilon(\theta)
& =
\sum_\mathbf{\tilde x} 
\nabla_\theta \PP_1^\theta(\mathbf{\tilde x}) \,
\frac{\partial}{\partial \PP_1(\mathbf{\tilde x})}
{\rm KL}(\PP_1 \,||\, \PP_\infty^\theta).
\end{align}

\subsubsection{Score matching (Fisher divergence)}

A particularly interesting objective function for energy-based models is
score matching \cite{ScoreMatching}, which has been subsequently used
for example in \cite{ZCLZ16,KinLe10,Hyv08,KLH09,SRBMF11,Vin11}.  Similar
to contrastive divergence, score matching is an objective function,
which we can avoid computationally intractable evaluation of expectation
with.

Specifically, Hyv\"arinen \cite{ScoreMatching} defines score matching as
\begin{align}
 {\rm FD}(\PP_{\rm target} \,||\, \PP_\theta)
 & = \int_\mathbf{x} \PP_{\rm target}(\mathbf{x})
 \left| \nabla_\mathbf{x} \log \PP_{\rm target}(\mathbf{x}) - \nabla_\mathbf{x} \log \PP_\theta(\mathbf{x}) \right|^2 \,
 d\mathbf{x},
\end{align}
which is also referred to as Fisher divergence \cite{Lyu09}.  Note that
the gradient is with respect to $\mathbf{x}$ and not $\theta$.  
Hyv\"arinen shows that
\begin{align}
 \theta^\star \equiv \argmin_\theta  {\rm FD}(\PP_{\rm target} \,||\, \PP_\theta)
\end{align}
is a consistent estimator \cite{ScoreMatching}.

A key property of score matching is that there is no need for calculating the partition function.  In particular,
\begin{align}
 \nabla_\mathbf{x} \log \PP_\theta(\mathbf{x})
 & = -\nabla_\mathbf{x} E_\theta(\mathbf{x}) - \nabla_\mathbf{x} \log \int_\mathbf{\tilde x} \exp(-E_\theta(\mathbf{\tilde x})) \, d\mathbf{\tilde x} \\
 & = -\nabla_\mathbf{x} E_\theta(\mathbf{x}).
\end{align}
Also, although $\nabla_\mathbf{\tilde x} \log \PP_{\rm
target}(\mathbf{x})$ might appear to be intractable, it has been shown
in \cite{ScoreMatching} that
\begin{align}
 \frac{1}{D} \sum_{\mathbf{x} \in {\cal D}} \sum_{i=1}^N
 \left(
 \frac{\partial \log \PP_\theta(\mathbf{x})^2}{\partial x_i^2}
 + \frac{1}{2}  \left( \frac{\partial \log \PP_\theta(\mathbf{x})}{\partial x_i} \right)^2
 \right)
 + {\rm const}
 \label{eq:FDestimator}
\end{align}
is asymptotically equivalent to ${\rm FD}(\PP_{\rm
target}\,||\,\PP_\theta)$ as $D\equiv|{\cal D}|\to\infty$, where ${\cal D}$
is the set of data (samples from $\PP_{\rm target}$), $N$ is the
dimension of $\mathbf{x}\in{\cal D}$, and ${\rm const}$ is the term
independent of $\theta$.

A limitation of the estimator in \eqref{eq:FDestimator} is that it
assumes among others that the variables $\mathbf{x}$ are continuous and
and $\PP_\theta(\cdot)$ is differentiable.  There has been prior work
for relaxing these assumptions.  For example, Hyv\"arinen studies an
extension for binary or non-negative data \cite{Hyv07}, and Kingma and
LaCun study an approach of adding Gaussian noise to data samples for
smoothness condition \cite{KinLe10}.

\subsection{A separate generator}

The learning rule that follows from maximization of loglikelihood (minimization of KL divergence) via a
gradient ascent method may be considered as decreasing the energy of
``positive'' or ``real'' samples that are generated according to a target
distribution $\PP_{\rm target}$ and increasing the energy of
``negative'' or ``fake'' samples that are generated according to the current model
(recall \eqref{eq:BM:grad-log}):
\begin{align}
\nabla_\theta f(\theta)
& = - \sum_\mathbf{\tilde x} \PP_{\rm target}(\mathbf{\tilde x}) \,
	\nabla_\theta E_\theta(\mathbf{\tilde x})
+ \sum_\mathbf{\tilde x} \PP_\theta(\mathbf{\tilde x}) \,
 \nabla_\theta E_\theta(\mathbf{\hat x}).
 \label{eq:minKL}
\end{align}
With this learning rule, one can let the model to be able to better
``discriminate between the positive examples from the original data and
the negative examples generated by sampling from the current density
estimate'' \cite{WZH03}.

Then an energy-based model may be considered as taking both the role of a
generator and the role of a discriminator in a generative adversarial
network (GAN) \cite{SchGAN,GAN}.  There is a line of work
\cite{DeepEnergyRL,DABHC17,ZML16,KimBen16} that prepares a separate
generator with an energy-based model.  If the separate generator allows
more efficient sampling than the energy-based model, the expectation
with respect to the distribution of the current generator ({\it i.e.},
$\sum_\mathbf{\tilde x} \PP_\theta(\mathbf{\tilde x}) \, \nabla_\theta
E_\theta(\mathbf{\hat x}$) in \eqref{eq:minKL}) can be more efficiently
evaluated \cite{KimBen16,DeepEnergyRL}.

\subsection{Mean-field Boltzmann machines}
\label{sec:BM:real:meanfield}

There are also approaches that avoid computationally expensive 
evaluation of expectation.  An example is a mean-field Boltzmann machine 
\cite{MeanFieldBM}, which can be used to approximate a Boltzmann machine. 
A mean-field Boltzmann machine ignores connections between units and 
chooses the real value $m_i$ for each unit $i$ in a way that the 
distribution defined through
\begin{align}
\QQ_\mathbf{m}(\mathbf{x}) & \equiv \prod_i m_i^{x_i} \, (1-m_i)^{1-x_i}
\end{align}
well approximates the distribution $\PP_\theta(\mathbf{x})$ of a 
Boltzmann machine in the sense of the KL divergence 
\cite{MeanFieldWelHin}.

\section{Other energy-based models}
\label{sec:BM:others}

Here we review stochastic models that are related to the Boltzmann machine.

\subsection{Markov random fields}

A Boltzmann machine is a Markov random field \cite{MRFbook} having a 
particular structure.  A Markov random field consists of a finite set of 
units similar to a Boltzmann machine.  Each unit of a Markov random 
field takes a value in a finite set.  The probability distribution of 
the values (configurations) of the Markov random field can be represented as
\begin{align}
\PP(\mathbf{x})
& = \frac{\exp(-E(\mathbf{x}))}
	{\displaystyle\sum_{\mathbf{\tilde x}} \exp(-E(\mathbf{\tilde x}))},
\end{align}
where the summation with respect to $\mathbf{\tilde x}$ is over all of
the possible configurations of the values in the finite set, for which
the Markov random field is defined.  The energy of a configuration is
defined as follows:
\begin{align}
E(\mathbf{x}) = \mathbf{w}^\top \, \boldsymbol{\phi}(\mathbf{x}),
\end{align}
where $\boldsymbol{\phi}(\mathbf{x})$ is a feature vector of $\mathbf{x}$.
A Markov random field is also called an undirected graphical model. 

\subsubsection{Boltzmann machine and Ising model}

A Markov random field is reduced to a Boltzmann machine when it has the 
following two properties. First, $\boldsymbol{\phi}(\mathbf{x})$ is a 
vector of monomials of degrees up to 2. Second, each unit takes a binary 
value.

An Ising model \cite{Ising} is essentially equivalent to a Boltzmann 
machine but the binary variable takes values in $\{-1,+1\}$.

\subsubsection{Higher-order Boltzmann machine}

A higher-order Boltzmann machine extends a Boltzmann machine by allowing 
$\boldsymbol{\phi}(\mathbf{x})$ to include monomials of degree greater 
than 2 \cite{HigherBM}. Each unit of a higher-order Boltzmann machine 
takes a binary value.

%\subsection{}
%conditional random field \cite{CRF,CRFintro} probably does not belong here

\subsection{Determinantal point process}

A determinantal point process (DPP) defines a probability distribution
over the subsets of a given ground set
\cite{DPP,DPP4ML,MR2018415,MR1989442}.  In our context, the ground set
$\mathcal{Y}$ can be considered as the set of all units:
\begin{align}
\mathcal{Y} & = \{1, 2, \ldots, N\}.
\end{align}
A subset $\mathcal{X}$ can be 
considered as a set of units that take the value 1:
\begin{align}
\mathcal{X} 
& = \{ i \mid x_i=1, i\in\mathcal{Y} \}.
\end{align}

A DPP can be characterized by a kernel $\mathbf{L}$, which is an 
$N\times N$ positive semi-definite matrix.  The probability that the 
subset $\mathcal{X}$ is selected ({\it i.e.}, the units in $\mathcal{X}$ 
take the value 1) is then given by
\begin{align}
\PP(\mathcal{X})
 = \frac{\det(\mathbf{L}_{\mathcal{X}})}{\det(\mathbf{L}+\mathbf{I})},
\label{eq:DPP}
\end{align}
where $\mathbf{L}_{\mathcal{X}}$ denotes the principal submatrix of 
$\mathbf{L}$ indexed by $\mathcal{X}$, and $\mathbf{I}$ is the $N\times 
N$ identity matrix.  A DPP can be seen as an energy-based model, whose energy is given 
by $E(\mathcal{X})=-\log\det(\mathbf{L}_{\mathcal{X}})$.

In general, the kernel $\mathbf{L}$ can be represented as
\begin{align}
\mathbf{L} &= \mathbf{B}^\top \, \mathbf{B}
\label{eq:L_factor}
\end{align} 
by the use of a $K\times N$ matrix $\mathbf{B}$, where $K$ is the rank 
of $\mathbf{L}$.  For $i\in[1,N]$, let $\mathbf{b}_i$ be the $i$-th 
column of $\mathbf{B}$, which may be understood as a feature vector of 
the $i$-th item (unit).  One can further decompose $\mathbf{b}_i$ into 
$\mathbf{b}_i = q_i \, \bm{\phi}_i$, where $||\bm{\phi}_i||=1$ and $q_i\ge 
0$.  Then we can write the $(i,j)$-th element of $\mathbf{L}$ as follows:
\begin{align}
\ell_{i,j}
 = (\mathbf{b}_i)^\top \, \mathbf{b}_j
 = q_i \, (\bm{\phi}_i)^\top \, \bm{\phi}_j \, q_j.
\end{align}
In particular,
\begin{align}
\PP(\{i\}) \sim \det(\mathbf{L}_{\{i\}}) = \ell_{i,i} = q_i^2,
\end{align}
so that $q_i$ can be understood as the "quality" of the $i$-th item.  
Specifically, given that exactly one item is selected, the conditional 
probability that the $i$-th item is selected is proportional to $q_i^2$.  Likewise,
\begin{align}
\PP(\{i, j\})
 \sim \det(\mathbf{L}_{\{i,j\}})
 = \det\left(\begin{array}{cc}
	q_i \, (\bm{\phi}_i)^\top \, \bm{\phi}_i \, q_i & q_i \, (\bm{\phi}_i)^\top \, \bm{\phi}_j \, q_j \\
	q_j \, (\bm{\phi}_j)^\top \, \bm{\phi}_i \, q_i & q_j \, (\bm{\phi}_j)^\top \, \bm{\phi}_j \, q_j
	\end{array}\right)
 = q_i^2 \, q_j^2 \, ( 1 - S_{i,j}^2),
\label{eq:DPP_Pij}
\end{align}
where
\begin{align}
S_{i,j} \equiv (\bm{\phi}_i)^\top \, \bm{\phi}_j
\end{align}
may be understood as the similarity between item $i$ and item $j$.  
Equation~\eqref{eq:DPP_Pij} implies that the similar items are unlikely 
to be selected together.  In general, a DPP tends to give high 
probability to diverse subsets of items having high quality.

Finally, we discuss a computational aspect of a DPP.  Let's compare the 
denominator of the right-hand side of \eqref{eq:DPP} against the 
corresponding denominator (partition function) of the Boltzmann machine 
in \eqref{eq:BM:prob}.  The partition function of the Boltzmann machine 
is a summation over $2^N$ terms, which suggest that we need $O(2^N)$ 
time to evaluate it.  On the other hand, the determinant of an $N\times 
N$ matrix can be computed in $O(N^3)$ time.

When $\mathbf{L}$ has a low rank ($K<N$), one can further reduce the 
computational complexity.  Let
\begin{align}
\mathbf{C} &= \mathbf{B} \, \mathbf{B}^\top
\end{align}
be the $K\times K$ positive semi-definite matrix, defined from 
\eqref{eq:L_factor}.  Because the eigenvalues 
$\{\lambda_1,\ldots,\lambda_K\}$ of $\mathbf{L}$ are eigenvalues of 
$\mathbf{C}$ and vice versa, we have
\begin{align}
\det(\mathbf{L}+\mathbf{I})
= \prod_{k=1}^K (\lambda_k+1)
= \det(\mathbf{C}+\mathbf{I}),
\end{align}
where the first identity matrix $\mathbf{I}$ is $N\times N$, and the 
second is $K\times K$.  Therefore, \eqref{eq:DPP} can be represented as follows:
\begin{align}
\mathcal{P}(\mathcal{X})
 & = \frac{\det(\mathbf{L}_{\mathcal{X}})}{\det(\mathbf{C}+\mathbf{I})}.
\label{eq:DPP-dual}
\end{align}
The denominator of this dual representation of a DPP can be evaluated in 
$O(K^3)$ time.

The DPP has been receiving increasing attention in machine learning, and 
various learning algorithms have been proposed in the literature 
\cite{EMFullRankDPP,LowRankDPP,kronDPP,DyDPP}.

\subsection{Gaussian Boltzmann machines}
\label{sec:BM:Gaussian}

Here we review energy based models that deal with real values with a particular 
focus on those models that are extended from Boltzmann machines.  
A standard approach to extend the Boltzmann machine to deal with real values is 
the use of a Gaussian unit \cite{GaussianBM,WRH04,GBRBM}.

\subsubsection{Gaussian Bernoulli restricted Boltzmann machines}

For real values $\mathbf{x}\in\RR^N$ and binary values
$\mathbf{h}\in\{0,1\}^M$, Krizhevsky studies the following energy
\cite{GBRBM}:
\begin{align}
E_\theta(\mathbf{x},\mathbf{h})
 = \sum_{i=1}^N \frac{(x_i-b_i^{\rm V})^2}{2\,\sigma_i^2}
	- \sum_{j=1}^M b_j^{\rm H} \, h_j
	- \sum_{i=1}^N \sum_{j=1}^M x_i \, \frac{w_{i,j}}{\sigma_i} \, h_j,
\label{eq:GBRBM-energy}
\end{align}
where $\theta\equiv(\mathbf{b}^{\rm V}, \mathbf{b}^{\rm H}, \mathbf{W}, 
\bm{\sigma})$ is the set of parameters.

Krizhevsky shows that the conditional probability of $\mathbf{x}$ given 
$\mathbf{h}$ has a normal distribution, and the conditional probability 
of $\mathbf{h}$ given $\mathbf{x}$ has a Bernoulli distribution 
\cite{GBRBM}.  Here, we re-derive them with our notations and in a simpler manner.
\begin{theorem}
Consider the energy given by \eqref{eq:GBRBM-energy}.  Then the elements 
of $\mathbf{x}$ are conditionally independent of each other given 
$\mathbf{h}$, and the elements of $\mathbf{h}$ are conditionally 
independent of each other given $\mathbf{x}$.  Also, the conditional probability density 
of $x_i$ given $\mathbf{h}$ is given by
\begin{align}
p_\theta^{(i)}(x_i\mid\mathbf{h})
& = \frac{1}{\sqrt{2\,\pi\,\sigma_i^2}} \exp\Bigg(
	- \frac{\Big(x_i - \big(b_i^{\rm V} + \sigma_i\sum_{j=1}^M w_{i,j}\,h_j\big)\Big)^2}{2\,\sigma_i^2}
	\Bigg)
\label{eq:GBRBM-gaussian}
\end{align}
for $x_i\in\RR$, and the conditional mass probability of $h_j$ for $h_j\in\{0,1\}$ given $\mathbf{x}$ is given by
\begin{align}
p_\theta^{(j)}(h_j\mid\mathbf{x})
& = \frac{
	\exp\bigg( \Big(b_j^{\rm H} + \sum_{i=1}^N x_i \, \frac{w_{i,j}}{\sigma_i} \Big) h_j \bigg)
	}{
	1 + \exp\Big(b_j^{\rm H} + \sum_{i=1}^N x_i \, \frac{w_{i,j}}{\sigma_i} \Big)
	}.
\label{eq:GBRBM-bernoulli}
\end{align}
\label{thrm:GBRBM}
\end{theorem}
\begin{proof}
First observe that
\begin{align}
\exp(-E_\theta(\mathbf{x},\mathbf{h}))
& = \exp( (\mathbf{b}^{\rm H})^\top \mathbf{h} )
 \, \prod_{i=1}^N \exp(-E_\theta^{(i)}(x_i,\mathbf{h})),
\intertext{where}
E_\theta^{(i)}(x_i,\mathbf{h})
& = \frac{(x_i-b_i^{\rm V})^2}{2\,\sigma_i^2}
	- \sum_{j=1}^M x_i \, \frac{w_{i,j}}{\sigma_i} \, h_j.
\end{align}
This means that the elements of $\mathbf{x}$ are conditionally 
independent of each other given $\mathbf{h}$. The conditional 
independence of the elements of $\mathbf{h}$ given $\mathbf{x}$ can be 
shown analogously.

Then we have
\begin{align}
p_\theta(\mathbf{x}\mid\mathbf{h})
& = \prod_{i=1}^N p_\theta^{(i)}(x_i\mid\mathbf{h}) \\
p_\theta(\mathbf{h}\mid\mathbf{x})
& = \prod_{j=1}^M p_\theta^{(j)}(h_j\mid\mathbf{x}),
\intertext{where}
p_\theta^{(i)}(x_i\mid\mathbf{h})
& \sim \exp\big(-E_\theta^{(i)}(x_i,\mathbf{h})\big) \\
& = \exp\Bigg(
	- \frac{x_i^2 - 2 \Big(b_i^{\rm V} + \sigma_i\sum_{j=1}^M w_{i,j}\,h_j + (b_i^{\rm V})^2\Big)}{2\,\sigma_i^2}
	\Bigg) \\
& \sim \exp\Bigg(
	- \frac{\Big(x_i - \big(b_i^{\rm V} + \sigma_i\sum_{j=1}^M w_{i,j}\,h_j\big)\Big)^2}{2\,\sigma_i^2}
	\Bigg)\\
p_\theta^{(j)}(h_j\mid\mathbf{x})
%& \sim \exp(-E_\theta^{(j)}(\mathbf{x},h_j) \\
& \sim \exp\bigg( \Big(b_j^{\rm H} + \sum_{i=1}^N x_i \, \frac{w_{i,j}}{\sigma_i} \Big) h_j \bigg)
\end{align}
for $x_i \in \RR$ and $h_j\in\{0,1\}$.  By taking into 
account the normalization for the total probability to become 1, we 
obtain \eqref{eq:GBRBM-gaussian} and \eqref{eq:GBRBM-bernoulli}.
\end{proof}

One might wonder why the product $\sigma_i \, w_{i,j}$ appears in 
\eqref{eq:GBRBM-gaussian}, and the quotient $w_{i,j}/\sigma_i$ appears in 
\eqref{eq:GBRBM-bernoulli}.  We can shed light on these expressions by 
studying natural parameters as in \cite{WRH04}.  The natural parameters of the normal 
distribution with mean $\mu$ and standard deviation $\sigma$ are $\mu/\sigma^2$ 
and $-1/(2\,\sigma^2)$.  Hence, the natural parameters in \eqref{eq:GBRBM-gaussian} are
\begin{align}
\frac{b_i^{\rm V}}{\sigma_i^2} + \sum_{j=1}^M \frac{w_{i,j}}{\sigma_i} h_j
\mbox{   and   }
-\frac{1}{2\,\sigma_i^2}.
\end{align}
Likewise, the natural parameter in \eqref{eq:GBRBM-bernoulli} is
\begin{align}
b_j^{\rm H} + \sum_{i=1}^N x_i \, \frac{w_{i,j}}{\sigma_i}.
\end{align}
Therefore, only the quotient $w_{i,j}/\sigma_i$ appears in natural parameters.

\subsubsection{Spike and slab restricted Boltzmann machines}

Courville et al.\ study a class of particularly structured higher-order 
Boltzmann machines with Gaussian units, which they refer to as spike and 
slab RBMs \cite{SpikeSlab,SpikeSlab2}.  For example, the energy may be 
represented as follows:
\begin{align}
E_\theta(\mathbf{x},\mathbf{h},\mathbf{S})
& = \sum_{i=1}^N \frac{\lambda_i}{2} x_i^2
	+ \sum_{j=1}^M \sum_{k=1}^K \frac{\alpha_{j,k}}{2} s_{j,k}^2
	- \sum_{j=1}^N b_j \, h_j
	- \sum_{i=1}^N \sum_{j=1}^M \sum_{k=1}^K w_{i,j,k} \, x_i \, h_j \, s_{j,k},
\end{align}
where $\mathbf{x}$ denotes real-valued visible values, $\mathbf{h}$ 
denotes binary hidden values (``spikes''), $\mathbf{S}$ denotes 
real-valued ``slabs,'' and 
$\theta\equiv(\bm{\lambda},\bm{\alpha},\mathbf{W},\mathbf{b})$ denotes 
the parameters.  The term $w_{i,j,k} \, x_i \, h_j \, s_{j,k}$ 
represents a three way interactions.

Similar to Theorem~\ref{thrm:GBRBM}, one can show that the elements of 
$\mathbf{x}$ are conditionally independent of each other and have normal 
distributions given $\mathbf{h}$ and $\mathbf{S}$:
\begin{align}
p(x_i \mid \mathbf{h}, \mathbf{S})
& \sim \exp\bigg(
	-\frac{\lambda_i}{2}
	\Big(x_i - \frac{1}{\lambda_i} \sum_{j=1}^M \sum_{k=1}^K w_{i,j,k} \, h_j \, s_{j,k} \Big)^2
	\bigg)
\end{align}
for $x_i\in\RR$.  Likewise, the elements of $\mathbf{S}$ are conditionally independent of each other and have normal 
distributions given $\mathbf{x}$ and $\mathbf{h}$:
\begin{align}
p(s_{j,k} \mid \mathbf{x}, \mathbf{h})
& \sim \exp\bigg(
	-\frac{\alpha_{j,k}}{2}
	\Big(s_{j,k} - \frac{1}{\alpha_i} \sum_{i=1}^N w_{i,j,k} \, x_i \, h_j \Big)^2
	\bigg)
\end{align}
for $s_{j,k}\in\RR$.  Given $\mathbf{x}$ and $\mathbf{S}$, the elements
of $\mathbf{h}$ are conditionally independent of each other and have
Bernoulli distributions:
\begin{align}
p(h_j \mid \mathbf{x}, \mathbf{S})
& \sim \exp\bigg(
	\Big(
	b_j + \sum_{i=1}^N \sum_{k=1}^K w_{i,j,k} \, x_i \, s_{j,k}
	\Big) h_j
	\bigg)
\end{align}
for $h_j\in\{0,1\}$.

It has been argued and experimentally confirmed that spike and slab RBMs 
can generate images with sharper boundaries than those generated by 
models with binary hidden units \cite{SpikeSlab, SpikeSlab2}.

\subsection{Using expected values to represent real values}

Here, we discuss the approach of using expected values given by the 
probability distribution defined by a Boltzmann machine.  Because a unit 
of a Boltzmann machine takes a binary value, 0 or 1, the expected value 
is in $[0, 1]$.  With appropriate scaling, any closed interval can be 
mapped to $[0,1]$.

\subsubsection{Expected values in visible units}

Recall, from Section~\ref{sec:BM:generative}, that a Boltzmann machine 
with parameter $\theta\equiv(\mathbf{b},\mathbf{W})$ defines a 
probability distribution over binary values:
\begin{align}
  \PP_\theta(\mathbf{x})
  & =
  \frac{\exp\left( -E_\theta(\mathbf{x}) \right)}
  {\displaystyle\sum_{\mathbf{\tilde x}} \exp\left( -E_\theta(\mathbf{\tilde x}) \right)},
\intertext{where}
  E_\theta(\mathbf{x})
  & \equiv
  - \mathbf{b}^\top \mathbf{x} - \mathbf{x}^\top\mathbf{W}\,\mathbf{x}
\end{align}
for $\mathbf{x} \in \{0, 1\}^N$.  The expected values can then be given as
\begin{align}
\E_\theta[\boldsymbol{X}] 
& = \sum_{\mathbf{\tilde x}} \mathbf{x} \, \PP_\theta(\mathbf{x}),
\end{align}
which take values in $[0,1]^N$.

A question is whether the values given by the expectation are suitable 
for a particular purpose.  In other words, how should we set $\theta$ so 
that the expectation gives suitable values for the purpose under 
consideration?

Here, we discuss a particular learning rule of simply using the values 
in $[0,1]^N$ in the learning rules for binary values.  In particular, 
a gradient ascent method for generative learning without hidden units 
is given by \eqref{eq:BM:grad_bi} and \eqref{eq:BM:grad_wij}.  A stochastic gradient method 
with mini-batches is then given by
\begin{align}
  b_i
  & \leftarrow
  b_i + \eta \, \left(\frac{1}{K}\sum_{k=1}^K X_i(\omega_k) -  \E_\theta[X_i]\right) \\
  w_{i,j}
  & \leftarrow
  w_{i,j} + \eta \, \left(\frac{1}{K}\sum_{k=1}^K X_i(\omega_k)\,X_j(\omega_k) -  \E_\theta[X_i\,X_j]\right),
\end{align}
where we take $K$ samples, $\boldsymbol{X}(\omega_1), \ldots, \boldsymbol{X}(\omega_K)$, at each 
step of the mini-batch stochastic gradient method.  As $K\to\infty$, these learning rules converge to 
the gradient ascent method given by \eqref{eq:BM:grad_bi} and \eqref{eq:BM:grad_wij}.  
If the real values under consideration are actually expectation of some binary random variables, 
we can justify a stochastic gradient method of updating $b_i$ based on a sampled real value $R_i(\omega)$ 
according to 
\begin{align}
  b_i
  & \leftarrow
  b_i + \eta \, \left(R_i(\omega) -  \E_\theta[X_i]\right).
\end{align}
However, the corresponding update of $w_{i,j}$ cannot be justified,
because $\E_\theta[X_i\,X_j]$ cannot be represented solely by
$\E_\theta[X_i]$ and $\E_\theta[X_j]$ unless $X_i$ and $X_j$ are
independent.  If $X_i$ and $X_j$ are independent, we should have
$w_{i,j}=0$.

In general, the use of expected values can be justified only for the
special case where the corresponding random variables are
conditionally independent of each other given input values.  The
simple discriminative model in Figure~\ref{fig:BM:simple} is an
example of such a special case.  In this case, one can apply a
stochastic gradient method of updating the parameters as follows:
\begin{align}
  b_j
  & \leftarrow
  b_j + \eta \, \left(R_j(\omega) -  \E_\theta[Y_j\,|\,\boldsymbol{X}(\omega)]\right) \\
  w_{i,j}
  & \leftarrow
  b_{i,j}
  + \eta \, \left( X_i(\omega) \, R_j(\omega) - X_i(\omega) \, \E_\theta[Y_j\,|\,\boldsymbol{X}(\omega)]\right)
\end{align}
for a sampled pair $(X_i(\omega), R_j(\omega))$, where $X_i(\omega)$
is an input binary value for $i\in[1,N_{\rm in}]$, and $R_j(\omega)$
is an output real value in $[0,1]$ for $j\in[1,N_{\rm out}]$.

\subsubsection{Expected values in hidden units}

Expected values are more often used for hidden units \cite{TRBM, RTRBM, 
ExpectedEnergyRL, BidirectionalDyBM} than for visible units.

Consider a Boltzmann machine with visible units and hidden units, which
have no connections between visible units or between hidden units
(namely, a restricted Boltzmann machine or RBM; see
Figure~\ref{fig:RBM}).  Let $\mathbf{b}^{\rm H}$ be the bias associated
with hidden units, $\mathbf{b}^{\rm V}$ be the bias associated with
visible units, and $\mathbf{W}$ be the weight between visible units and
hidden units. Then, given the visible values $\mathbf{x}$, the hidden
values $\mathbf{h}$ are conditionally independent of each other, and we
can represent the conditional expected value of the $j$-th hidden unit
as follows:
\begin{align}
m_j(\mathbf{x}) & = \frac{1}{1+\exp\big(-b_j(\mathbf{x}) \big)},
\label{eq:BM:mj}
\intertext{where}
b_j(\mathbf{x}) 
& \equiv b_j^{\rm H} + \mathbf{x}^\top \mathbf{W}_{:,j}.
\end{align}

\begin{figure}
\centering
\includegraphics[width=0.5\linewidth]{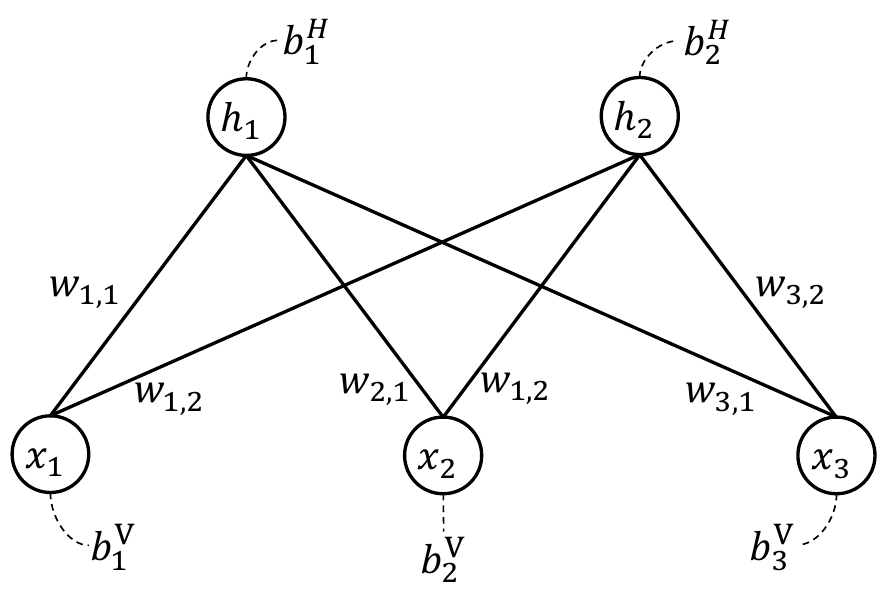}
\caption{A restricted Boltzmann machine}
\label{fig:RBM}
\end{figure}

Because the energy is a linear function of $\mathbf{h}$, we can represent 
the expected energy, where the expectation is with respect to the conditional 
distribution of the hidden values $\boldsymbol{H}$ given the visible values $\mathbf{x}$, 
using the conditional expected values of hidden units:
\begin{align}
\E[ E_\theta(\mathbf{x},\boldsymbol{H}) ]
& = - (\mathbf{b}^{\rm V})^\top \mathbf{x}
	- \mathbf{x}^\top \mathbf{W} \, \mathbf{m}
	-  (\mathbf{b}^{\rm H})^\top \mathbf{m} \\
& = - (\mathbf{b}^{\rm V})^\top \mathbf{x}
	- \mathbf{b}(\mathbf{x})^\top \mathbf{m}(\mathbf{x}).
\label{eq:expected-energy}
\end{align}
Notice that the distribution of visible values is then given by
\begin{align}
\PP_\theta(\mathbf{x})
& = \frac{1}{1+\exp\Big(-\E\big[ E_\theta(\mathbf{x},\boldsymbol{H}) \big] \Big)}.
\end{align}

This expected energy may be compared against the corresponding free energy:
\begin{align}
F_\theta(\mathbf{x})
& = - \log \sum_{\mathbf{\tilde h}} \exp( - E_\theta(\mathbf{x}, \mathbf{\tilde h})),
\end{align}
where
\begin{align}
E_\theta(\mathbf{x}, \mathbf{\tilde h}) 
& = - (\mathbf{b}^{\rm V})^\top \mathbf{x}
	- (\mathbf{b}^{\rm H})^\top \mathbf{\tilde h}
	- \mathbf{x}^\top \mathbf{W} \, \mathbf{\tilde h}.
\end{align}
The free energy can be represented as follows:
\begin{align}
F_\theta(\mathbf{x})
 & = - \log \exp\big( (\mathbf{b}^{\rm V})^\top \mathbf{x} \big)
	\sum_{\mathbf{\tilde h}} \exp\big( \mathbf{b}(\mathbf{x})^\top \mathbf{\tilde h}\big) \\
 & = - (\mathbf{b}^{\rm V})^\top \mathbf{x}
	- \log \prod_{j=1}^M \big( 1 + \exp( b_j(\mathbf{x}) ) \big) \\
 & = - (\mathbf{b}^{\rm V})^\top \mathbf{x}
	- \sum_{j=1}^M \log\big( 1 + \exp( b_j(\mathbf{x}) )\big).
\label{eq:free-energy-RBM}
\end{align}

\begin{theorem}
\begin{align}
\E_\theta[ E_\theta(\mathbf{x},\boldsymbol{H}) ]
 & = F_\theta(\mathbf{x}) - \E_\theta[\log \PP_\theta(\boldsymbol{H}\mid \mathbf{x})],
\end{align}
where $-\E_\theta[\log \PP_\theta(\boldsymbol{H}\mid \mathbf{x})]$ 
is the entropy of the conditional distribution of hidden values $\boldsymbol{H}$ given visible values $\mathbf{x}$.
\end{theorem}
\begin{proof}
By Corollary \ref{corollary:BM:independence}, hidden values are conditionally independent of each other given visible values:
\begin{align}
\log \PP_\theta(\boldsymbol{H}\mid \mathbf{x})
& = \sum_{j=1}^M \log \PP_\theta(H_j\mid \mathbf{x}).
\end{align}
Then, using the notation in \eqref{eq:BM:mj}, we obtain
\begin{align}
E_\theta[\log \PP_\theta(\boldsymbol{H}\mid \mathbf{x})]
& = \sum_{j=1}^M \bigg(
	m_j(\mathbf{x}) \, \log m_j(\mathbf{x}) 
	+ (1 - m_j(\mathbf{x})) \, \log (1 - m_j(\mathbf{x})) 
	\bigg) \\
& = \sum_{j=1}^M
\bigg(
m_j \log\frac{\exp(b_j(\mathbf{x}))}{1+\exp(b_j(\mathbf{x}))}
+ (1-m_j) \log\frac{1}{1+\exp(b_j(\mathbf{x}))}
\bigg) \\
& = 
\mathbf{b}(\mathbf{x})^\top \mathbf{m}(\mathbf{x})
- \sum_{j=1}^M \log(1 + \exp(b_j(\mathbf{x}))).
\label{eq:BM:entropy}
\end{align}
The theorem now follows by adding \eqref{eq:expected-energy} and \eqref{eq:BM:entropy}, 
comparing it against \eqref{eq:free-energy-RBM}.
\end{proof}

Figure~\ref{fig:expected-energy} compares expected energy and free energy.  
Specifically, the blue curve shows the value of
\begin{align}
\frac{b_j^{\rm H} + \mathbf{x}^\top \mathbf{W}_{:,j}}{1+\exp\big(-b_j^{\rm H} - \mathbf{x}^\top \mathbf{W}_{:,j}\big)},
\end{align}
which appears in the expression of expected energy \eqref{eq:expected-energy},
and the green curve shows the value of
\begin{align}
\log\big( 1 + \exp( b_j^{\rm H} + \mathbf{x}^\top\mathbf{W}_{:,j})\big),
\end{align}
which appears in the expression of free energy
\eqref{eq:free-energy-RBM}.  The difference between the two curves is
largest ($\log 2\approx 0.30$) when $b_j^{\rm
H}+\mathbf{x}^\top\mathbf{W}_{:,j}=0$.  The two curves are essentially
indistinguishable when $|b_j^{\rm H}+\mathbf{x}^\top\mathbf{W}_{:,j}|$
is sufficiently large.  This suggests that the Boltzmann machine with
expected energy is different from the corresponding Boltzmann machine
(with free energy), but the two models have some similarity.

\begin{figure}[t]
\centering
\includegraphics[width=0.5\linewidth]{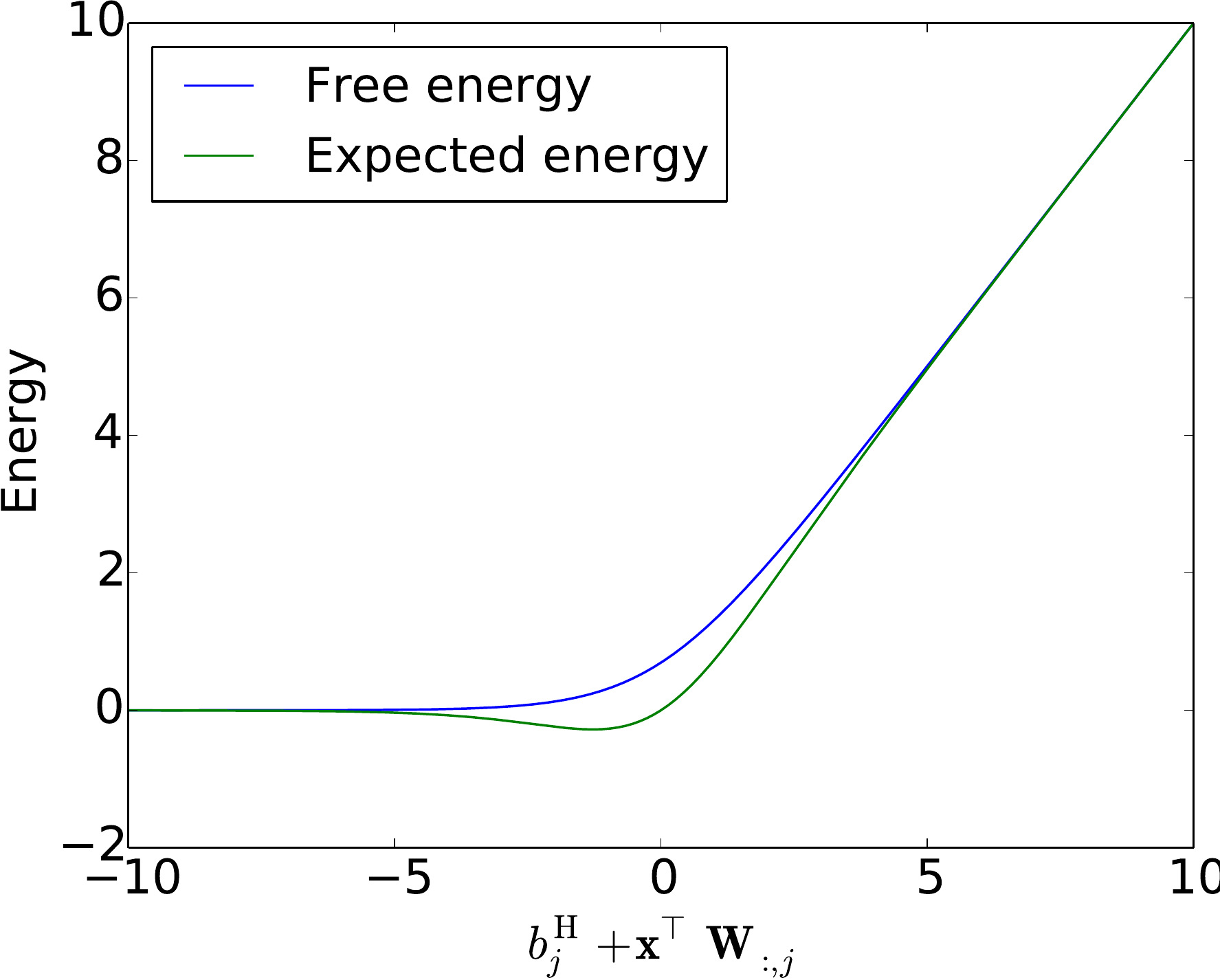}
\caption{Comparison between expected energy 
\eqref{eq:expected-energy} and free energy 
\eqref{eq:free-energy-RBM} associated with the $j$-th hidden unit 
of a restricted Boltzmann machine.}
\label{fig:expected-energy}
\end{figure}

\section{Non-probabilistic energy-based models}
\label{sec:BM:energy}

So far, we have studied probabilistic models.  A difficulty that we 
often face with probabilistic models is in high computational complexity 
for evaluating the partition function, or the normalization for the 
total probability to become 1. The probabilistic models that we have 
studied can be turned into (non-probabilistic) energy-based models.  
Such energy-based models do not require normalization, which is 
attractive from computational point of view.  A difficulty in 
energy-based models is in designing appropriate objective functions.  In this 
section, we will study (non-probabilistic) energy-based models through 
an example \cite{OLM07}.  For more details, see a tutorial by LeCun 
\cite{LeCun06}.

\subsection{Objective functions for energy-based models}

In learning an energy-based model, one needs to carefully design an objective (loss) 
function \cite{LeCun06,EBMloss} in a way that minimizing the objective 
function leads to desired results.  We can then optimize the parameters 
$\theta$ of the energy-based model by minimizing the objective function.  
The energy-based model with the optimized parameters $\theta$ should 
give low energy to desirable values of the variables of the energy-based 
model and high energy to other values.

Here, we consider energy-based models with input and output.  Let 
$E_\theta(\mathbf{x},y)$ be the energy for a pair of an input 
$\mathbf{x}$ and an output $y$.  An energy-based model with parameter 
$\theta$ gives
\begin{align}
y^\star = \argmin_y E_\theta(\mathbf{x},y)
\end{align}
as the output for input $\mathbf{x}$.  A desirable pair of input and 
output should give lower energy than undesirable ones.

When we optimize an energy-based model, minimizing the energy of a given 
data is usually inappropriate.  In particular, such an objective function may 
be unbounded.  It may not distinguish two patterns, one is good and the 
other is very good, as both of the two patterns have the minimum energy.

An objective function of an energy-based model should have a contrastive 
term, which naturally appear in the objective function of a 
probabilistic model ({\it i.e.}, the KL divergence or the log likelihood). 
 For example, recall from \eqref{eq:BM:prob} that the average negative 
log likelihood of a set of patterns $\mathcal{D}$ with respect to a 
Boltzmann machine is given by
\begin{align}
- \frac{1}{|\mathcal{D}|} \sum_{\mathbf{x}\in\mathcal{D}} \log \PP_\theta(\mathbf{x})
& = 
\frac{1}{|\mathcal{D}|} \sum_{\mathbf{x}\in\mathcal{D}} E_\theta(\mathbf{x})
- \log \sum_{\mathbf{\tilde x}} \exp(-E_\theta(\mathbf{\tilde x})).
\label{eq:energy:LL}
\end{align}
The second term of the right-hand side of \eqref{eq:energy:LL} is a 
contrastive term.  In particular, to minimize this objective function, we should not 
only reduce the energy of the patterns in $\mathbf{D}$ but also increase 
the energy of the patterns not in $\mathbf{D}$.  Recall also 
Figure~\ref{fig:BM:SGD}.  However, the contrastive term involves the 
summation over exponentially many ($2^N$) patterns, and is the source of 
computational intractability.  In designing an objective function for an 
energy-based model, we want to design a contrastive term that can be 
more efficiently evaluated.  We will see an example in the following.

\subsection{An example of face detection with pose estimation}

Osadchy et al.\ study an energy-based approach for classifying images 
into ``face'' or ``non-face'' and estimating the facial pose at the same 
time \cite{OLM07}.  Let $\mathbf{x}$ be a vector representing an image, 
$y$ be a variable indicating ``face'' or ``non-face,'' and $\mathbf{z}$ 
be a vector representing a facial pose.  They consider an energy 
function of the following form:
\begin{align}
E_\theta( y, \mathbf{z}, \mathbf{x} )
& = y \, || G_\theta(\mathbf{x}) - F(\mathbf{z}) || + (1-y) \, T,
\label{eq:face_energy}
\end{align}
where $G_\theta(\cdot)$ is a convolutional neural network (CNN), having 
parameter $\theta$, that maps an $\mathbf{x}$ into a lower dimensional 
vector, $F(\cdot)$ is an arbitrarily defined function that maps a 
$\mathbf{z}$ on to a manifold embedded in the low dimensional space of 
the output of the CNN, and $T$ is a constant (see Figure~\ref{fig:BM:face-arch}).

\begin{figure}
\centering
\includegraphics[width=\linewidth]{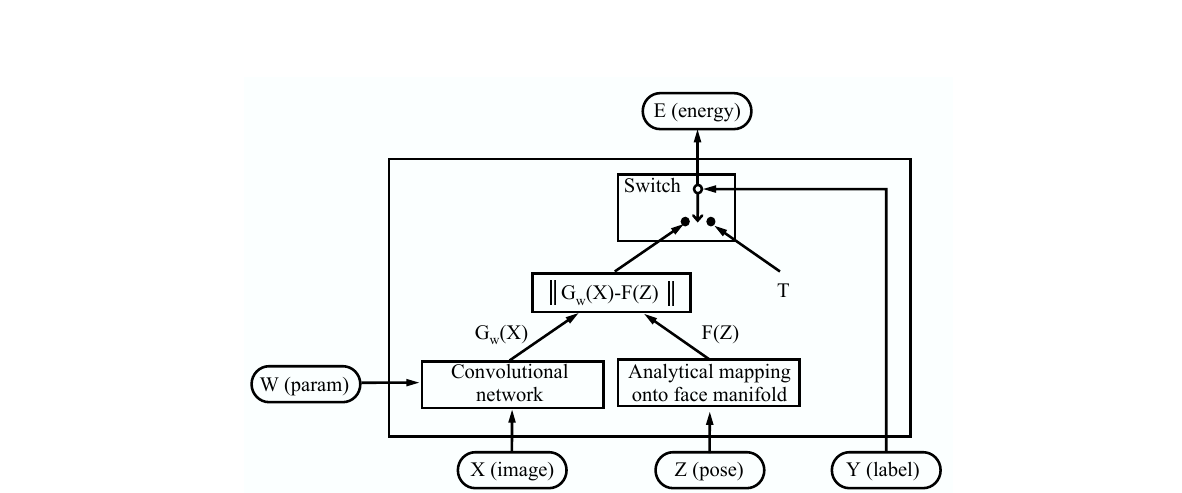}
\caption{An architecture of the energy \eqref{eq:face_energy}: Figure 4 from \cite{OLM07}, where $W$ is used to denote parameters $\theta$.}
\label{fig:BM:face-arch}
\end{figure}

An image $\mathbf{x}$ is classified as ``face'' ($y=1$) if
\begin{align}
\min_\mathbf{\tilde z} E_\theta( 1, \mathbf{\tilde z}, \mathbf{x} )
<
\min_\mathbf{\tilde z} E_\theta( 0, \mathbf{\tilde z}, \mathbf{x} ),
\end{align}
which is equivalent to
\begin{align}
\min_\mathbf{\tilde z} || G_\theta(\mathbf{x}) - F(\mathbf{\tilde z}) ||
< T.
\end{align}

We want to learn the parameters $\theta$ from a given training dataset.  
A training dataset $\mathcal{D}$ consists of two subsets.  The first 
subset $\mathbf{D}_1$ includes facial images with their facial poses 
$(\mathbf{z},\mathbf{x})$.  The second subset $\mathbf{D}_2$ includes 
non-facial images $\mathbf{x}$.

It would be computationally intractable to maximize the log-likelihood 
of the training dataset with respect to the probability distribution 
that could be defined with the energy $\eqref{eq:face_energy}$ through
\begin{align}
\PP_\theta(y,\mathbf{z},\mathbf{x}) 
\sim
\exp(-E_\theta(y,\mathbf{z},\mathbf{x})).
\end{align}

Osadchy et al.\ instead minimizes the following objective function \cite{OLM07}:
\begin{align}
L_\theta(\mathcal{D})
& = \frac{1}{|\mathcal{D}_1|} 
	\sum_{(\mathbf{z},\mathbf{x})\in{\cal D}_1} 
	E_\theta(1,\mathbf{z},\mathbf{x})^2
+ \frac{\kappa}{\mathcal{D}_2|} 
	\sum_{\mathbf{x}\in{\cal D}_2}
	\exp\Big(-\min_\mathbf{\tilde z} E_\theta(1,\mathbf{\tilde z},\mathbf{x})\Big),
\label{eq:loss_face}
\end{align}
where $\kappa$ is a positive constant (hyper-parameter).  The first term of 
the right-hand side of \eqref{eq:loss_face} is the average squared 
energy of the ``face'' samples.  Hence, minimizing the energy of 
``face'' samples leads to minimizing the objective function.  The second 
(contrastive) term involves the energy of ``non-face'' samples when they 
are classified as ``face'' ($y=1$), where the face pose is set optimally 
so that the energy is minimized.  By minimizing the second term, one can 
expect to make the energy of those ``non-face'' samples with the face 
label ($y=1$) greater than the corresponding energy with the non-face 
label ($y=0$).

\section{Conclusion}
\label{sec:BM:conclusion}

We have reviewed basic properties of Boltzmann machines with a 
particular focus on those that are relevant for gradient-based 
approaches to learning their parameters.  As it turns out that exact 
learning is in general intractable, we have discussed general approaches 
to approximation as well as tractable alternative, namely energy-based 
models.

This paper has covered limited perspectives of Boltzmann machines and 
energy-based models. For example, we have only briefly discussed restricted 
Boltzmann machines (RBMs), which deserves intensive review in its own.  
In fact, RBMs and deep Boltzmann machines \cite{DBM} are the topics that 
are covered in the second part of an IJCAI-17 
tutorial\footnote{https://researcher.watson.ibm.com/researcher/view\_gro 
up.php?id=7834}. This paper does not cover Boltzmann machines for 
time-series, which are reviewed in a companion paper \cite{BMsurvey3}.

The use of energy is also restrictive in this paper and in line with the
tutorial by LeCun \cite{LeCun06}. Our perspective on energy-based
learning in the tutorial is, however, broader than what is suggested by
LeCun.  We may use energy for other purposes.  An example is the use of
free energy \cite{FreeEnergyRL,FreeEnergyActorCritic,FreeEnergyPOMDP} or
expected energy \cite{ExpectedEnergyRL} to approximate the Q-function in
reinforcement learning.  This energy-based reinforcement learning is the
topic covered in the fourth part of the IJCAI-17 tutorial.  A key aspect
of energy-based reinforcement learning is that the energy used to
approximate the Q-function naturally defines a probability distribution,
which is used for exploration.  Thus, energy-based models that allow
efficient sampling have the advantage in sampling the actions to explore
in reinforcement learning.  There is recent work on the use of deep
neural networks to approximate the Q-function in the framework of
energy-based reinforcement learning \cite{DeepEnergyRL}, where they use
a separate sampling network for efficiency.

\section*{Acknowledgments}
This work was supported by JST CREST Grant Number JPMJCR1304, Japan.  The author thanks Chen-Feng Liu and Samuel Laferriere for pointing out typographical errors in the original version.

\bibliographystyle{plain}
\bibliography{gradient,neuro,boltzmann,dpp,energy,RL,dybm,rnn}

\end{document}